\documentclass{article}

\usepackage{times}
\usepackage{graphicx} 
\usepackage{subfigure}

\usepackage{natbib}

\usepackage{algorithm}
\usepackage{algorithmic}
\usepackage[algo2e,linesnumbered,lined,boxed,vlined]{algorithm2e}

\usepackage{amsfonts}
\usepackage{amssymb}

\usepackage{hyperref}

\usepackage[accepted]{icml2016}

\usepackage{comment}
\usepackage{amsthm}
\usepackage{mathtools}
\usepackage{tabularx}
\usepackage{multirow}

\newtheorem{proposition}{Proposition}

\def\b{\ensuremath\boldsymbol}

\usepackage{lastpage}
\usepackage{fancyhdr}
\usepackage{url}
\icmltitlerunning{Unsupervised and Supervised Principal Component Analysis: Tutorial}

\begin{document}

\twocolumn[
\icmltitle{Unsupervised and Supervised Principal Component Analysis: Tutorial}

\icmlauthor{Benyamin Ghojogh}{bghojogh@uwaterloo.ca}
\icmladdress{Department of Electrical and Computer Engineering, 
\\Machine Learning Laboratory, University of Waterloo, Waterloo, ON, Canada}
\icmlauthor{Mark Crowley}{mcrowley@uwaterloo.ca}
\icmladdress{Department of Electrical and Computer Engineering, 
\\Machine Learning Laboratory, University of Waterloo, Waterloo, ON, Canada}

\icmlkeywords{Tutorial, Principal Component Analysis}

\vskip 0.3in
]

\begin{abstract}
This is a detailed tutorial paper which explains the Principal Component Analysis (PCA), Supervised PCA (SPCA), kernel PCA, and kernel SPCA. We start with projection, PCA with eigen-decomposition, PCA with one and multiple projection directions, properties of the projection matrix, reconstruction error minimization, and we connect to autoencoder. Then, PCA with singular value decomposition, dual PCA, and kernel PCA are covered. SPCA using both scoring and Hilbert-Schmidt independence criterion are explained. Kernel SPCA using both direct and dual approaches are then introduced. We cover all cases of projection and reconstruction of training and out-of-sample data. Finally, some simulations are provided on Frey and AT\&T face datasets for verifying the theory in practice.
\end{abstract}

\section{Introduction}

Assume we have a dataset of \textit{instances} or \textit{data points} $\{(\b{x}_i, \b{y}_i)\}_{i=1}^n$ with sample size $n$ and dimensionality $\b{x}_i  \in \mathbb{R}^d$ and $\b{y}_i \in \mathbb{R}^\ell$. 
The $\{\b{x}_i\}_{i=1}^n$ are the input data to the model and the $\{\b{y}_i\}_{i=1}^n$ are the observations (labels).
We define $\mathbb{R}^{d \times n} \ni \b{X} := [\b{x}_1, \dots, \b{x}_n]$ and $\mathbb{R}^{\ell \times n} \ni \b{Y} := [\b{y}_1, \dots, \b{y}_n]$.
We can also have an out-of-sample data point, $\b{x}_t \in \mathbb{R}^d$, which is not in the training set. If there are $n_t$ out-of-sample data points, $\{\b{x}_{t,i}\}_1^{n_t}$, we define $\mathbb{R}^{d \times n_t} \ni \b{X}_t := [\b{x}_{t,1}, \dots, \b{x}_{t,n_t}]$.
Usually, the data points exist on a subspace or sub-manifold. Subspace or manifold learning tries to learn this sub-manifold \cite{ghojogh2019feature}.

Principal Component Analysis (PCA) \cite{jolliffe2011principal} is a very well-known and fundamental linear method for subspace learning and dimensionality reduction \cite{friedman2001elements}. This method, which is also used for feature extraction \cite{ghojogh2019feature}, was first proposed by Pearson in 1901 \cite{pearson1901liii}.
In order to learn a nonlinear sub-manifold, kernel PCA was proposed by \cite{scholkopf1997kernel,scholkopf1998nonlinear}. It maps the data to high dimensional feature space hoping to fall on a linear manifold in that space.  

PCA and kernel PCA are unsupervised methods for subspace learning. To use the class labels in PCA, supervised PCA was proposed \cite{bair2006prediction} which scores the features of the $\b{X}$ and reduces the features before applying PCA. This type of SPCA were mostly used in bioinformatics \cite{ma2011principal}.

Afterwards, another type of SPCA \cite{barshan2011supervised} was proposed which has a very solid theory and PCA is actually a special case of it when we the labels are not used. 
This SPCA also has dual and kernel SPCA.
It is noteworthy that parametric PCA \cite{levada2020parametric} has also been proposed recently. 

PCA and SPCA have had many applications for example eigenfaces \cite{turk1991eigenfaces,turk1991face} and kernel eigenfaces \cite{yang2000face} for face recognition and detecting orientation of image using PCA \cite{mohammadzade2017critical}.
There exist many other applications of PCA and SPCA in the literature.
In this paper, we explain the theory of PCA, kernel SPCA, SPCA, and kernel SPCA and provide some simulations for verifying the theory in practice.

\section{Principal Component Analysis}

\subsection{Projection Formulation}

\subsubsection{A Projection Point of View}

Assume we have a data point $\b{x} \in \mathbb{R}^d$. We want to project this data point onto the vector space spanned by $p$ vectors $\{\b{u}_1, \dots, \b{u}_p\}$ where each vector is $d$-dimensional and usually $p \ll d$. We stack these vectors column-wise in matrix $\b{U} = [\b{u}_1, \dots, \b{u}_p] \in \mathbb{R}^{d \times p}$. In other words, we want to project $\b{x}$ onto the column space of $\b{U}$, denoted by $\mathbb{C}\text{ol}(\b{U})$.

The projection of $\b{x} \in \mathbb{R}^d$ onto $\mathbb{C}\text{ol}(\b{U}) \in \mathbb{R}^p$ and then its representation in the $\mathbb{R}^d$ (its reconstruction) can be seen as a linear system of equations:
\begin{align}\label{equation_projection}
\mathbb{R}^d \ni \widehat{\b{x}} := \b{U \beta},
\end{align}
where we should find the unknown coefficients $\b{\beta} \in \mathbb{R}^p$. 

If the $\b{x}$ lies in the $\mathbb{C}\text{ol}(\b{U})$ or $\textbf{span}\{\b{u}_1, \dots, \b{u}_p\}$, this linear system has exact solution, so $\widehat{\b{x}} = \b{x} = \b{U \beta}$. However, if $\b{x}$ does not lie in this space, there is no any solution $\b{\beta}$ for $\b{x} = \b{U \beta}$ and we should solve for projection of $\b{x}$ onto $\mathbb{C}\text{ol}(\b{U})$ or $\textbf{span}\{\b{u}_1, \dots, \b{u}_p\}$ and then its reconstruction. In other words, we should solve for Eq. (\ref{equation_projection}). In this case, $\widehat{\b{x}}$ and $\b{x}$ are different and we have a residual:
\begin{align}\label{equation_residual_1}
\b{r} = \b{x} - \widehat{\b{x}} = \b{x} - \b{U \beta},
\end{align}
which we want to be small. As can be seen in Fig. \ref{figure_residual_and_space}, the smallest residual vector is orthogonal to $\mathbb{C}\text{ol}(\b{U})$; therefore:
\begin{align}
\b{x} - \b{U\beta} \perp \b{U} &\implies \b{U}^\top (\b{x} - \b{U \beta}) = 0, \nonumber \\
& \implies \b{\beta} = (\b{U}^\top \b{U})^{-1} \b{U}^\top \b{x}. \label{equation_beta}
\end{align}
It is noteworthy that the Eq. (\ref{equation_beta}) is also the formula of coefficients in linear regression \cite{friedman2001elements} where the input data are the rows of $\b{U}$ and the labels are $\b{x}$; however, our goal here is different. Nevertheless, in Section \ref{section_PCA_reconstruction}, some similarities of PCA and regression will be introduced.

\begin{figure}[!t]
\centering
\includegraphics[width=2.2in]{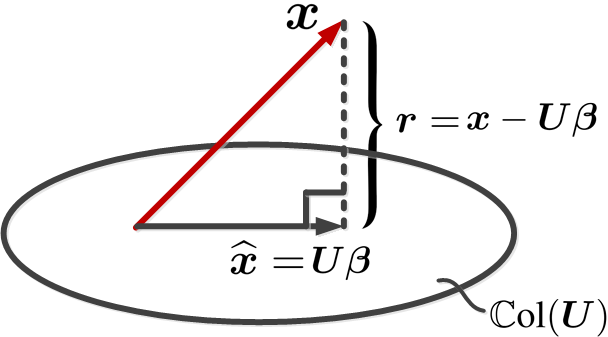}
\caption{The residual and projection onto the column space of $\b{U}$.}
\label{figure_residual_and_space}
\end{figure}

Plugging Eq. (\ref{equation_beta}) in Eq. (\ref{equation_projection}) gives us:
\begin{align*}
\widehat{\b{x}} = \b{U} (\b{U}^\top \b{U})^{-1} \b{U}^\top \b{x}.
\end{align*}
We define:
\begin{align}\label{equation_hat_matrix}
\mathbb{R}^{d \times d} \ni \b{\Pi} := \b{U} (\b{U}^\top \b{U})^{-1} \b{U}^\top,
\end{align}
as ``projection matrix'' because it projects $\b{x}$ onto $\mathbb{C}\text{ol}(\b{U})$ (and reconstructs back).
Note that $\b{\Pi}$ is also referred to as the ``hat matrix'' in the literature because it puts a hat on top of $\b{x}$.

If the vectors $\{\b{u}_1, \dots, \b{u}_p\}$ are orthonormal (the matrix $\b{U}$ is orthogonal), we have $\b{U}^\top = \b{U}^{-1}$ and thus $\b{U}^\top \b{U} = \b{I}$. Therefore, Eq. (\ref{equation_hat_matrix}) is simplified:
\begin{align}
& \b{\Pi} = \b{U} \b{U}^\top.
\end{align}
So, we have:
\begin{align}\label{equation_x_hat}
\widehat{\b{x}} = \b{\Pi}\, \b{x} = \b{U} \b{U}^\top \b{x}.
\end{align}

\begin{figure}[!t]
\centering
\includegraphics[width=2.8in]{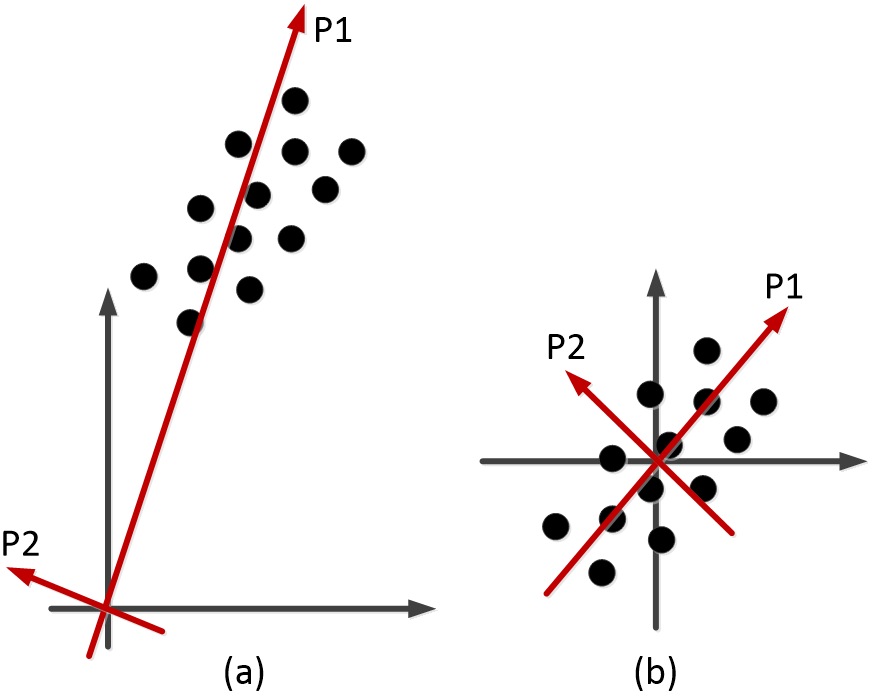}
\caption{The principal directions P1 and P2 for (a) non-centered and (b) centered data. As can be seen, the data should be centered for PCA.}
\label{figure_centered_data}
\end{figure}

\subsubsection{Projection and Reconstruction in PCA}

The Eq. (\ref{equation_x_hat}) can be interpreted in this way: The $\b{U}^\top \b{x}$ projects $\b{x}$ onto the row space of $\b{U}$, i.e., $\mathbb{C}\text{ol}(\b{U}^\top)$ (projection onto a space spanned by $d$ vectors which are $p$-dimensional). We call this projection, ``projection onto the PCA subspace''. It is ``subspace'' because we have $p \leq d$ where $p$ and $d$ are dimensionality of PCA subspace and the original $\b{x}$, respectively. Afterwards, $\b{U} (\b{U}^\top \b{x})$ projects the projected data back onto the column space of $\b{U}$, i.e., $\mathbb{C}\text{ol}(\b{U})$ (projection onto a space spanned by $p$ vectors which are $d$-dimensional). We call this step ``reconstruction from the PCA'' and we want the residual between $\b{x}$ and its reconstruction $\widehat{\b{x}}$ to be small.

If there exist $n$ training data points, i.e., $\{\b{x}_i\}_{i=1}^n$, the projection of a training data point $\b{x}$ is:
\begin{align}\label{equation_projection_onePoint}
\mathbb{R}^{p} \ni \widetilde{\b{x}} := \b{U}^\top \breve{\b{x}},
\end{align}
where:
\begin{align}\label{equation_mean_removed_x}
&\mathbb{R}^{d} \ni \breve{\b{x}} := \b{x} - \b{\mu}_x,
\end{align}
is the centered data point and:
\begin{align}
&\mathbb{R}^{d} \ni \b{\mu}_x := \frac{1}{n} \sum_{i=1}^n \b{x}_i,
\end{align}
is the mean of training data points.
The reconstruction of a training data point $\b{x}$ after projection onto the PCA subspace is:
\begin{align}\label{equation_reconstruction_onePoint}
\mathbb{R}^{d} \ni \widehat{\b{x}} := \b{U}\b{U}^\top \breve{\b{x}} + \b{\mu}_x = \b{U}\widetilde{\b{x}} + \b{\mu}_x,
\end{align}
where the mean is added back because it was removed before projection.

Note that in PCA, all the data points should be centered, i.e., the mean should be removed first. The reason is shown in Fig. \ref{figure_centered_data}.
In some applications, centering the data does not make sense. For example, in natural language processing, the data are text and centering the data makes some negative measures which is non-sense for text. Therefore, data is not sometimes centered and PCA is applied on the non-centered data. This method is called Latent Semantic Indexing (LSI) or Latent Semantic Analysis (LSA) \cite{dumais2004latent}.

If we stack the $n$ data points column-wise in a matrix $\b{X} = [\b{x}_1, \dots, \b{x}_n] \in \mathbb{R}^{d \times n}$, we first center them:
\begin{align}\label{equation_centered_training_data}
\mathbb{R}^{d \times n} \ni \breve{\b{X}} := \b{X} \b{H} = \b{X} - \b{\mu}_x,
\end{align}
where $\breve{\b{X}} = [\breve{\b{x}}_1, \dots, \breve{\b{x}}_n] = [\b{x}_1 - \b{\mu}_x, \dots, \b{x}_n - \b{\mu}_x]$ is the centered data and: 
\begin{align}
\mathbb{R}^{n \times n} \ni \b{H} := \b{I} - (1/n) \b{1}\b{1}^\top,
\end{align}
is the centering matrix.
See Appendix \ref{section_appendix_centering} for more details about the centering matrix.

The projection and reconstruction, Eqs. (\ref{equation_projection_onePoint}) and (\ref{equation_reconstruction_onePoint}), for the whole training data are:
\begin{align}
&\mathbb{R}^{p \times n} \ni \widetilde{\b{X}} := \b{U}^\top \breve{\b{X}}, \label{equation_projection_severalPoint} \\
&\mathbb{R}^{d \times n} \ni \widehat{\b{X}} := \b{U}\b{U}^\top \breve{\b{X}} + \b{\mu}_x = \b{U}\widetilde{\b{X}} + \b{\mu}_x, \label{equation_reconstruction_severalPoint}
\end{align}
where $\widetilde{\b{X}} = [\widetilde{\b{x}}_1, \dots, \widetilde{\b{x}}_n]$ and $\widehat{\b{X}} = [\widehat{\b{x}}_1, \dots, \widehat{\b{x}}_n]$ are the projected data onto PCA subspace and the reconstructed data, respectively.

We can also project a new data point onto the PCA subspace for $\b{X}$ where the new data point is not a column of $\b{X}$. In other words, the new data point has not had impact in constructing the PCA subspace. This new data point is also referred to as ``test data point'' or ``out-of-sample data'' in the literature.
The Eq. (\ref{equation_projection_severalPoint}) was for projection of $\b{X}$ onto its PCA subspace. If $\b{x}_t$ denotes an out-of-sample data point, its projection onto the PCA subspace ($\widetilde{\b{x}}_t$) and its reconstruction ($\widehat{\b{x}}_t$) are:
\begin{align}
&\mathbb{R}^{p} \ni \widetilde{\b{x}}_t = \b{U}^\top \breve{\b{x}}_t, \label{equation_outOfSample_projection_PCA} \\
&\mathbb{R}^{d} \ni \widehat{\b{x}}_t = \b{U} \b{U}^\top \breve{\b{x}}_t + \b{\mu}_x = \b{U} \widetilde{\b{x}}_t + \b{\mu}_x, \label{equation_outOfSample_reconstruct_PCA}
\end{align}
where:
\begin{align}
\mathbb{R}^{d} \ni \breve{\b{x}}_t := \b{x}_t - \b{\mu}_x,
\end{align}
is the centered out-of-sample data point which is centered using the mean of training data. Note that for centering the out-of-sample data point(s), we should use the mean of the training data and not the out-of-sample data.

If we consider the $n_t$ out-of-sample data points, $\mathbb{R}^{d \times n_t} \ni \b{X}_t = [\b{x}_{t,1}, \dots, \b{x}_{t,n_t}]$, all together, the projection and reconstruction of them are:
\begin{align}
&\mathbb{R}^{p \times n_t} \ni \widetilde{\b{X}}_t = \b{U}^\top \breve{\b{X}}_t, \\
&\mathbb{R}^{d \times n_t} \ni \widehat{\b{X}}_t = \b{U} \b{U}^\top \breve{\b{X}}_t + \b{\mu}_x = \b{U} \widetilde{\b{X}}_t + \b{\mu}_x, 
\end{align}
respectively, where:
\begin{align}
\mathbb{R}^{d \times n_t} \ni \breve{\b{X}}_t := \b{X}_t - \b{\mu}_x.
\end{align}

\subsection{PCA Using Eigen-Decomposition}

\subsubsection{Projection Onto One Direction}

In Eq. (\ref{equation_reconstruction_onePoint}), if $p=1$, we are projecting $\b{x}$ onto only one vector $\b{u}$ and reconstruct it. If we ignore adding the mean back, we have:
\begin{align*}
\widehat{\b{x}} = \b{u}\b{u}^\top \breve{\b{x}}.
\end{align*}
The squared length (squared $\ell_2$-norm) of this reconstructed vector is:
\begin{align}
&||\widehat{\b{x}}||_2^2 = ||\b{u}\b{u}^\top \breve{\b{x}}||_2^2 = (\b{u}\b{u}^\top \breve{\b{x}})^\top (\b{u}\b{u}^\top \breve{\b{x}}) \nonumber \\
& = \breve{\b{x}}^\top \b{u} \underbrace{\b{u}^\top \b{u}}_{1} \b{u}^\top \breve{\b{x}} \overset{(a)}{=} \breve{\b{x}}^\top \b{u}\, \b{u}^\top \breve{\b{x}} \overset{(b)}{=} \b{u}^\top \breve{\b{x}}\, \breve{\b{x}}^\top \b{u}, \label{equation_x_hat_length_squared}
\end{align}
where $(a)$ is because $\b{u}$ is a unit (normal) vector, i.e., $\b{u}^\top \b{u} = ||\b{u}||_2^2 = 1$, and $(b)$ is because $\breve{\b{x}}^\top \b{u} = \b{u}^\top \breve{\b{x}} \in \mathbb{R}$.

Suppose we have $n$ data points $\{\b{x}_i\}_{i=1}^n$ where $\{\breve{\b{x}}_i\}_{i=1}^n$ are the centered data.
The summation of the squared lengths of their projections $\{\widehat{\b{x}}_i\}_{i=1}^n$ is:
\begin{align}\label{equation_sum_projected}
\sum_{i=1}^n ||\widehat{\b{x}}_i||_2^2 \overset{(\ref{equation_x_hat_length_squared})}{=} \sum_{i=1}^n \b{u}^\top \breve{\b{x}}_i\, \breve{\b{x}}_i^\top \b{u} = \b{u}^\top \Big(\sum_{i=1}^n \breve{\b{x}}_i\, \breve{\b{x}}_i^\top\Big) \b{u}.
\end{align}
Considering $\breve{\b{X}} = [\breve{\b{x}}_1, \dots, \breve{\b{x}}_n] \in \mathbb{R}^{d \times n}$, we have:
\begin{align}\label{equation_covariance_matrix}
\mathbb{R}^{d \times d} \ni \b{S} &:= \sum_{i=1}^n \breve{\b{x}}_i\, \breve{\b{x}}_i^\top = \breve{\b{X}} \breve{\b{X}}^\top 
\overset{(\ref{equation_centered_training_data})}{=} \b{X}\b{H} \b{H}^\top\b{X}^\top \nonumber \\
&\overset{(\ref{equation_centeringMatrix_is_symmetric})}{=} \b{X}\b{H} \b{H}\b{X}^\top \overset{(\ref{equation_centeringMatrix_is_idempotent})}{=} \b{X}\b{H}\b{X}^\top, 
\end{align}
where $\b{S}$ is called the ``covariance matrix''. If the data were already centered, we would have $\b{S} = \b{X} \b{X}^\top$.

Plugging Eq. (\ref{equation_covariance_matrix}) in Eq. (\ref{equation_sum_projected}) gives us:
\begin{align}
\sum_{i=1}^n ||\widehat{\b{x}}_i||_2^2 = \b{u}^\top \b{S} \b{u}.
\end{align}
Note that we can also say that $\b{u}^\top \b{S} \b{u}$ is the variance of the projected data onto PCA subspace. In other words, $\b{u}^\top \b{S} \b{u} = \mathbb{V}\text{ar}(\b{u}^\top \breve{\b{X}})$. This makes sense because when some non-random thing (here $\b{u}$) is multiplied to the random data (here $\breve{\b{X}}$), it will have squared (quadratic) effect on variance, and $\b{u}^\top \b{S} \b{u}$ is quadratic in $\b{u}$.

Therefore, $\b{u}^\top \b{S} \b{u}$ can be interpreted in two ways: (I) the squared length of reconstruction and (II) the variance of projection.

We want to find a projection direction $\b{u}$ which maximizes the squared length of reconstruction (or variance of projection):
\begin{equation}
\begin{aligned}
& \underset{\b{u}}{\text{maximize}}
& & \b{u}^\top \b{S} \b{u}, \\
& \text{subject to}
& & \b{u}^\top \b{u} = 1,
\end{aligned}
\end{equation}
where the constraint ensures that the $\b{u}$ is a unit (normal) vector as we assumed beforehand.

Using Lagrange multiplier \cite{boyd2004convex}, we have:
\begin{align*}
\mathcal{L} = \b{u}^\top \b{S} \b{u} - \lambda (\b{u}^\top \b{u} - 1),
\end{align*}
Taking derivative of the Lagrangian and setting it to zero gives:
\begin{align}\label{equation_pca_eigendecomposition_1}
& \mathbb{R}^p \ni \frac{\partial \mathcal{L}}{\partial \b{u}} = 2\b{S} \b{u} - 2\lambda \b{u} \overset{\text{set}}{=} 0 \implies \b{S} \b{u} = \lambda \b{u}. 
\end{align}
The Eq. (\ref{equation_pca_eigendecomposition_1}) is the eigen-decomposition of $\b{S}$ where $\b{u}$ and $\lambda$ are the leading eigenvector and eigenvalue of $\b{S}$, respectively \cite{ghojogh2019eigenvalue}. Note that the leading eigenvalue is the largest one. The reason of being leading is that we are maximizing in the optimization problem. 
As a conclusion, if projecting onto one PCA direction, the PCA direction $\b{u}$ is the leading eigenvector of the covariance matrix.
Note that the ``PCA direction'' is also called ``principal direction'' or ``principal axis'' in the literature. The dimensions (features) of the projected data onto PCA subspace are called ``principal components''.

\subsubsection{Projection Onto Span of Several Directions}

In Eq. (\ref{equation_reconstruction_onePoint}) or (\ref{equation_reconstruction_severalPoint}), if $p>1$, we are projecting $\breve{\b{x}}$ or $\breve{\b{X}}$ onto PCA subspace with dimensionality more than one and then reconstruct back. If we ignore adding the mean back, we have:
\begin{align*}
\widehat{\b{X}} = \b{U} \b{U}^\top \breve{\b{X}}.
\end{align*}
It means that we project every column of $\breve{\b{X}}$, i.e., $\breve{\b{x}}$, onto a space spanned by the $p$ vectors $\{\b{u}_1, \dots, \b{u}_p\}$ each of which is $d$-dimensional. Therefore, the projected data are $p$-dimensional and the reconstructed data are $d$-dimensional.

The squared length (squared Frobenius Norm) of this reconstructed matrix is:
\begin{align*}
||\widehat{\b{X}}||_F^2 &= ||\b{U}\b{U}^\top \breve{\b{X}}||_F^2 = \textbf{tr}\big((\b{U}\b{U}^\top \breve{\b{X}})^\top (\b{U}\b{U}^\top \breve{\b{X}})\big) \\
& = \textbf{tr}(\breve{\b{X}}^\top \b{U} \underbrace{\b{U}^\top \b{U}}_{\b{I}} \b{U}^\top \breve{\b{X}}) \overset{(a)}{=} \textbf{tr}(\breve{\b{X}}^\top \b{U} \b{U}^\top \breve{\b{X}}) \\
& \overset{(b)}{=} \textbf{tr}(\b{U}^\top \breve{\b{X}} \breve{\b{X}}^\top \b{U}),
\end{align*}
where $\textbf{tr}(.)$ denotes the trace of matrix, $(a)$ is because $\b{U}$ is an orthogonal matrix (its columns are orthonormal), and $(b)$ is because $\textbf{tr}(\breve{\b{X}}^\top \b{U} \b{U}^\top \breve{\b{X}}) = \textbf{tr}(\breve{\b{X}} \breve{\b{X}}^\top \b{U} \b{U}^\top) = \textbf{tr}(\b{U}^\top \breve{\b{X}} \breve{\b{X}}^\top \b{U})$.
According to Eq. (\ref{equation_covariance_matrix}), the $\b{S} = \breve{\b{X}} \breve{\b{X}}^\top$ is the covariance matrix; therefore:
\begin{align}
||\widehat{\b{X}}||_F^2 = \textbf{tr}(\b{U}^\top \b{S}\, \b{U}).
\end{align}

We want to find several projection directions $\{\b{u}_1, \dots, \b{u}_p\}$, as columns of $\b{U} \in \mathbb{R}^{d \times p}$, which maximize the squared length of reconstruction (or variance of projection):
\begin{equation}
\begin{aligned}
& \underset{\b{U}}{\text{maximize}}
& & \textbf{tr}(\b{U}^\top \b{S}\, \b{U}), \\
& \text{subject to}
& & \b{U}^\top \b{U} = \b{I},
\end{aligned}
\end{equation}
where the constraint ensures that the $\b{U}$ is an orthogonal matrix as we assumed beforehand.

Using Lagrange multiplier \cite{boyd2004convex}, we have:
\begin{align*}
\mathcal{L} = \textbf{tr}(\b{U}^\top \b{S}\, \b{U}) - \textbf{tr}\big(\b{\Lambda}^\top (\b{U}^\top \b{U} - \b{I})\big),
\end{align*}
where $\b{\Lambda} \in \mathbb{R}^{p \times p}$ is a diagonal matrix $\textbf{diag}([\lambda_1, \dots, \lambda_p]^\top)$ including the Lagrange multipliers.
\begin{align}
& \mathbb{R}^{d \times p} \ni \frac{\partial \mathcal{L}}{\partial \b{U}} = 2\b{S} \b{U} - 2\b{U} \b{\Lambda} \overset{\text{set}}{=} 0 \nonumber \\ 
& \implies \b{S} \b{U} = \b{U} \b{\Lambda}. \label{equation_pca_eigendecomposition_2}
\end{align}
The Eq. (\ref{equation_pca_eigendecomposition_2}) is the eigen-decomposition of $\b{S}$ where the columns of $\b{U}$ and the diagonal of $\b{\Lambda}$ are the eigenvectors and eigenvalues of $\b{S}$, respectively \cite{ghojogh2019eigenvalue}. The eigenvectors and eigenvalues are sorted from the leading (largest eigenvalue) to the trailing (smallest eigenvalue) because we are maximizing in the optimization problem.
As a conclusion, if projecting onto the PCA subspace or $\textbf{span}\{\b{u}_1, \dots, \b{u}_p\}$, the PCA directions $\{\b{u}_1, \dots, \b{u}_p\}$ are the sorted eigenvectors of the covariance matrix of data $\b{X}$. 

\subsection{Properties of $\b{U}$}

\subsubsection{Rank of the Covariance Matrix}

We consider two cases for $\breve{\b{X}} \in \mathbb{R}^{d \times n}$:

\begin{enumerate}
\item If the original dimensionality of data is greater than the number of data points, i.e., $d \geq n$: 
In this case, $\textbf{rank}(\breve{\b{X}}) = \textbf{rank}(\breve{\b{X}}^\top) \leq n$. Therefore, $\textbf{rank}(\b{S}) = \textbf{rank}(\breve{\b{X}}\breve{\b{X}}^\top) \leq \min\big(\textbf{rank}(\breve{\b{X}}), \textbf{rank}(\breve{\b{X}}^\top)\big)-1 = n-1$. Note that $-1$ is because the data are centered. For example, if we only have one data point, it becomes zero after centering and the rank should be zero.
\item If the original dimensionality of data is less than the number of data points, i.e., $d \leq n-1$ (the $-1$ again is because of centering the data): 
In this case, $\textbf{rank}(\breve{\b{X}}) = \textbf{rank}(\breve{\b{X}}^\top) \leq d$. Therefore, $\textbf{rank}(\b{S}) = \textbf{rank}(\breve{\b{X}}\breve{\b{X}}^\top) \leq \min\big(\textbf{rank}(\breve{\b{X}}), \textbf{rank}(\breve{\b{X}}^\top)\big) = d$. 
\end{enumerate}
So, we either have $\textbf{rank}(\b{S}) \leq n-1$ or $\textbf{rank}(\b{S}) \leq d$.

\subsubsection{Truncating $\b{U}$}

Consider the following cases:

\begin{figure}[!t]
\centering
\includegraphics[width=3in]{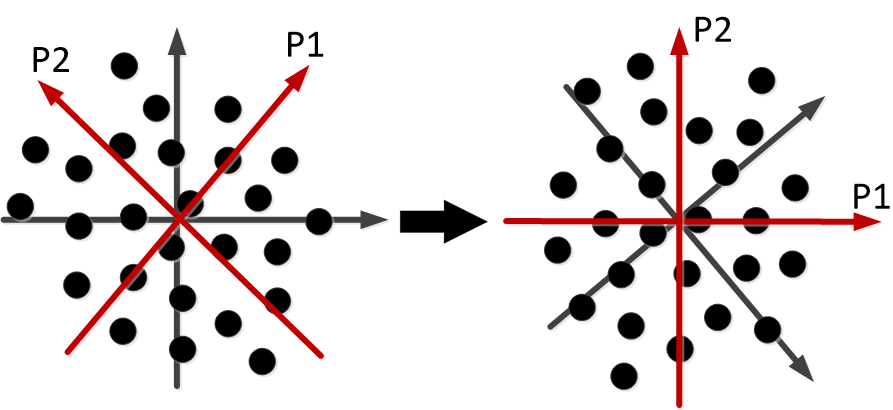}
\caption{Rotation of coordinates because of PCA.}
\label{figure_coordinate_rotation}
\end{figure}

\begin{enumerate}
\item If $\textbf{rank}(\b{S}) = d$: 
we have $p=d$ (we have $d$ non-zero eigenvalues of $\b{S}$), so that $\b{U} \in \mathbb{R}^{d \times d}$. It means that the dimensionality of the PCA subspace is $d$, equal to the dimensionality of the original space. Why does this happen? That is because $\textbf{rank}(\b{S}) = d$ means that the data are spread wide enough in all dimensions of the original space up to a possible rotation (see Fig. \ref{figure_coordinate_rotation}). Therefore, the dimensionality of PCA subspace is equal to the original dimensionality; however, PCA might merely rotate the coordinate axes. In this case, $\b{U} \in \mathbb{R}^{d \times d}$ is a square orthogonal matrix so that $\mathbb{R}^{d \times d} \ni \b{U}\b{U}^\top = \b{U}\b{U}^{-1} = \b{I}$ and $\mathbb{R}^{d \times d} \ni \b{U}^\top\b{U} = \b{U}^{-1}\b{U} = \b{I}$ because $\textbf{rank}(\b{U}) = d$, $\textbf{rank}(\b{U}\b{U}^\top) = d$, and $\textbf{rank}(\b{U}^\top\b{U}) = d$.
That is why in the literature, PCA is also referred to as coordinate rotation. 
\item If $\textbf{rank}(\b{S}) < d$ and $n > d$: 
it means that we have enough data points but the data points exist on a subspace and do not fill the original space wide enough in every direction. In this case, $\b{U} \in \mathbb{R}^{d \times p}$ is not square and $\textbf{rank}(\b{U}) = p < d$ (we have $p$ non-zero eigenvalues of $\b{S}$). Therefore, $\mathbb{R}^{d \times d} \ni \b{U}\b{U}^\top \neq \b{I}$ and $\mathbb{R}^{p \times p} \ni \b{U}^\top \b{U} = \b{I}$ because $\textbf{rank}(\b{U}) = p$, $\textbf{rank}(\b{U}\b{U}^\top) = p < d$, and $\textbf{rank}(\b{U}^\top\b{U}) = p$.
\item If $\textbf{rank}(\b{S}) \leq n-1 < d$: 
it means that we do not have enough data points to properly represent the original space and the points have an ``intrinsic dimensionality''. For example, we have two three-dimensional points which are one a two-dimensional line (subspace). So, similar to previous case, the data points exist on a subspace and do not fill the original space wide enough in every direction. The discussions about $\b{U}$, $\b{U}\b{U}^\top$, and $\b{U}^\top\b{U}$ are similar to previous case.
\end{enumerate}

Note that we might have $\textbf{rank}(\b{S}) = d$ and thus $\b{U} \in \mathbb{R}^{d \times d}$ but want to ``truncate'' the matrix $\b{U}$ to have $\b{U} \in \mathbb{R}^{d \times p}$. Truncating $\b{U}$ means that we take a subset of best (leading) eigenvectors rather than the whole $d$ eigenvectors with non-zero eigenvalues. In this case, again we have $\b{U}\b{U}^\top \neq \b{I}$ and $\b{U}^\top \b{U} = \b{I}$.
The intuition of truncating is this: the variance of data might be noticeably smaller than another direction; in this case, we can only keep the $p<d$ top eigenvectors (PCA directions) and ``ignore'' the PCA directions with smaller eigenvalues to have $\b{U} \in \mathbb{R}^{d \times p}$. Figure \ref{figure_PCA_oneDirection} illustrates this case for a 2D example.
Note that truncating can also be done, when $\b{U} \in \mathbb{R}^{d \times p}$, to have $\b{U} \in \mathbb{R}^{d \times q}$ where $p$ is the number of non-zero eigenvalues of $\b{S}$ and $q < p$.

From all the above analyses, we conclude that as long as the columns of the matrix $\b{U} \in \mathbb{R}^{d \times p}$ are orthonormal, we always have $\b{U}^\top \b{U} = \b{I}$ regardless of the value $p$. If the orthogonal matrix $\b{U}$ is not truncated and thus is a square matrix, we also have $\b{U} \b{U}^\top = \b{I}$.

\begin{figure}[!t]
\centering
\includegraphics[width=1.6in]{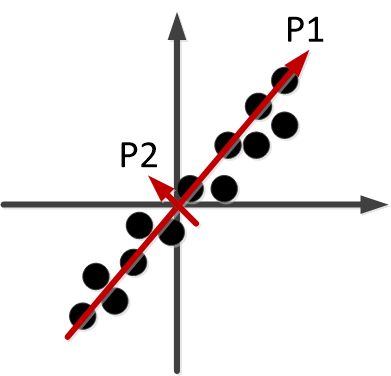}
\caption{A 2D example where the data is almost on a line and the second principal direction is very small and can be ignored.}
\label{figure_PCA_oneDirection}
\end{figure}

\subsection{Reconstruction Error in PCA}\label{section_PCA_reconstruction}

\subsubsection{Reconstruction in Linear Projection}

If we center the data, the Eq. (\ref{equation_residual_1}) becomes $\b{r} = \breve{\b{x}} - \widehat{\b{x}}$ because the reconstructed data will also be centered according to Eq. (\ref{equation_reconstruction_onePoint}). 
According to Eqs. (\ref{equation_residual_1}), (\ref{equation_mean_removed_x}), and (\ref{equation_reconstruction_onePoint}), we have:
\begin{align}
\b{r} = \b{x} - \widehat{\b{x}} = \breve{\b{x}} + \b{\mu}_x - \b{U} \b{U}^\top \breve{\b{x}} - \b{\mu}_x = \breve{\b{x}} - \b{U} \b{U}^\top \breve{\b{x}}.
\end{align}
Figure \ref{figure_reconstruction_error} shows the projection of a two-dimensional point (after the data being centered) onto the first principal direction, its reconstruction, and its reconstruction error. As can be seen in this figure, the reconstruction error is different from least square error in linear regression.

\begin{figure}[!t]
\centering
\includegraphics[width=3.1in]{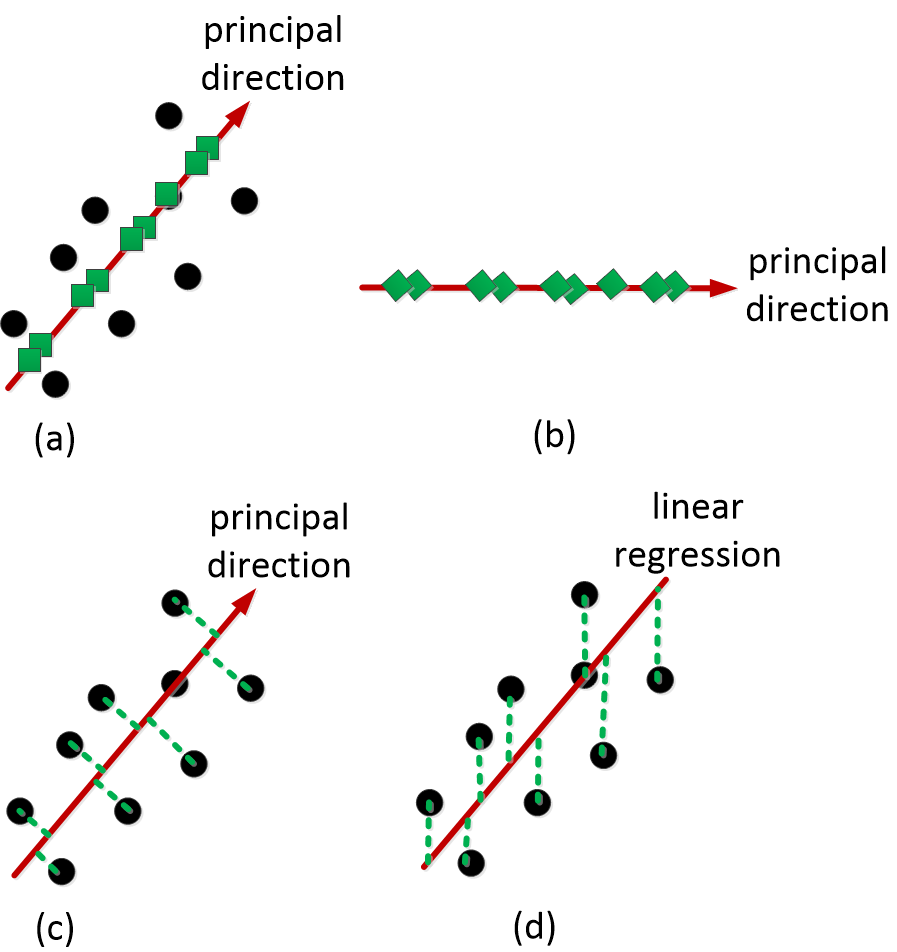}
\caption{(a) Projection of the black circle data points onto the principal direction where the green square data points are the projected data. (b) The reconstruction coordinate of the data points. (c) The reconstruction error in PCA. (d) The least square error in linear regression.}
\label{figure_reconstruction_error}
\end{figure}

For $n$ data points, we have:
\begin{align}
\b{R} &:= \b{X} - \widehat{\b{X}} = \breve{\b{X}} + \b{\mu}_x - \b{U} \b{U}^\top \breve{\b{X}} - \b{\mu}_x \nonumber \\
&= \breve{\b{X}} - \b{U} \b{U}^\top \breve{\b{X}},
\end{align}
where $\mathbb{R}^{d \times n} \ni \b{R} = [\b{r}_1, \dots, \b{r}_n]$ is the matrix of residuals.

If we want to minimize the reconstruction error subject to the orthogonality of the projection matrix $\b{U}$, we have:
\begin{equation}
\begin{aligned}
& \underset{\b{U}}{\text{minimize}}
& & ||\breve{\b{X}} - \b{U}\b{U}^\top\breve{\b{X}}||_F^2, \\
& \text{subject to}
& & \b{U}^\top \b{U} = \b{I}.
\end{aligned}
\end{equation}
The objective function can be simplified:
\begin{align*}
&||\breve{\b{X}} - \b{U}\b{U}^\top\breve{\b{X}}||_F^2 \\
&= \textbf{tr}\big((\breve{\b{X}} - \b{U}\b{U}^\top\breve{\b{X}})^\top (\breve{\b{X}} - \b{U}\b{U}^\top\breve{\b{X}})\big) \\
&= \textbf{tr}\big((\breve{\b{X}}^\top - \breve{\b{X}}^\top\b{U}\b{U}^\top) (\breve{\b{X}} - \b{U}\b{U}^\top\breve{\b{X}})\big) \\
&= \textbf{tr}(\breve{\b{X}}^\top\breve{\b{X}}-2\breve{\b{X}}^\top\b{U}\b{U}^\top\breve{\b{X}} + \breve{\b{X}}^\top \b{U} \underbrace{\b{U}^\top \b{U}}_{\b{I}} \b{U}^\top \breve{\b{X}}) \\
&= \textbf{tr}(\breve{\b{X}}^\top\breve{\b{X}}-\breve{\b{X}}^\top\b{U}\b{U}^\top\breve{\b{X}}) \\
&= \textbf{tr}(\breve{\b{X}}^\top\breve{\b{X}})-\textbf{tr}(\breve{\b{X}}^\top\b{U}\b{U}^\top\breve{\b{X}}) \\
&= \textbf{tr}(\breve{\b{X}}^\top\breve{\b{X}})-\textbf{tr}(\breve{\b{X}}\breve{\b{X}}^\top\b{U}\b{U}^\top).
\end{align*}
Using Lagrange multiplier \cite{boyd2004convex}, we have:
\begin{align*}
\mathcal{L} = &\,\textbf{tr}(\breve{\b{X}}^\top\breve{\b{X}})-\textbf{tr}(\breve{\b{X}}\breve{\b{X}}^\top\b{U}\b{U}^\top) \\
&+ \textbf{tr}\big(\b{\Lambda}^\top (\b{U}^\top \b{U} - \b{I})\big),
\end{align*}
where $\b{\Lambda} \in \mathbb{R}^{p \times p}$ is a diagonal matrix $\textbf{diag}([\lambda_1, \dots, \lambda_p]^\top)$ containing the Lagrange multipliers.
Equating the derivative of Lagrangian to zero gives:
\begin{align}
& \mathbb{R}^{d \times p} \ni \frac{\partial \mathcal{L}}{\partial \b{U}} = -2\breve{\b{X}}\breve{\b{X}}^\top \b{U} + 2\b{U} \b{\Lambda} \overset{\text{set}}{=} 0 \nonumber \\ 
& \implies \breve{\b{X}}\breve{\b{X}}^\top \b{U} = \b{U} \b{\Lambda}, \nonumber \\
& \overset{(\ref{equation_covariance_matrix})}{\implies} \b{S} \b{U} = \b{U} \b{\Lambda},
\end{align}
which is again the eigenvalue problem \cite{ghojogh2019eigenvalue} for the covariance matrix $\b{S}$. We had the same eigenvalue problem in PCA. Therefore, \textit{PCA subspace is the best linear projection in terms of reconstruction error. In other words, PCA has the least squared error in reconstruction.}

\subsubsection{Reconstruction in Autoencoder}

\begin{figure}[!t]
\centering
\includegraphics[width=2.3in]{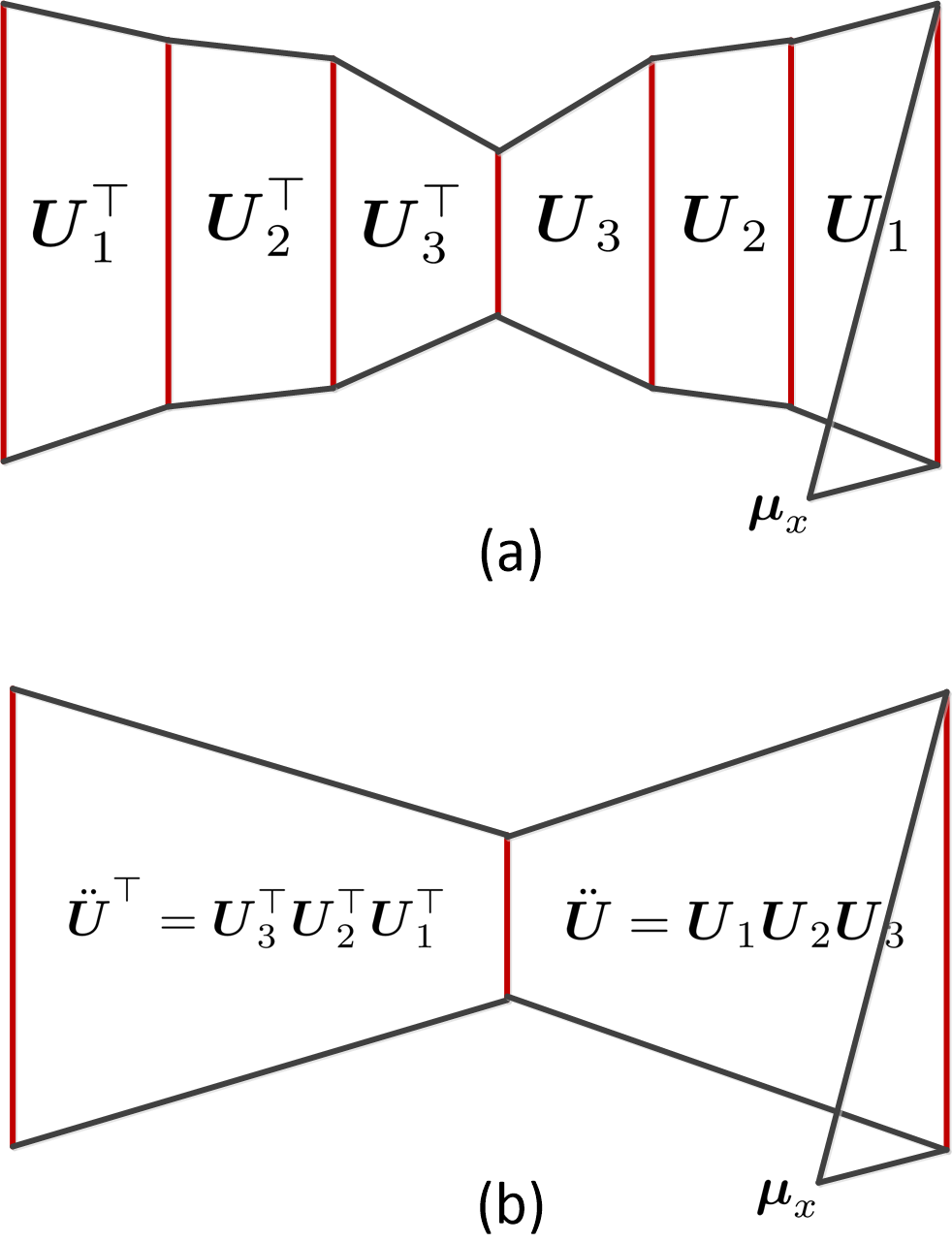}
\caption{(a) An example of autoencoder with five hidden layers and linear activation functions, and (b) its reduction to an autoencoder with one hidden layer.}
\label{figure_autoencoder}
\end{figure}

We saw that PCA is the best in reconstruction error for \textit{linear} projection. If we have $m>1$ successive linear projections, the reconstruction is:
\begin{align}
\widehat{\b{X}} = \underbrace{\b{U}_1 \cdots \b{U}_m}_\text{reconstruct} \underbrace{\b{U}_m^\top \cdots \b{U}_1^\top}_\text{project} \breve{\b{X}} + \b{\mu}_x,
\end{align}
which can be seen as an undercomplete \textit{autoencoder} \cite{goodfellow2016deep} with $2m$ layers without activation function (or with identity activation functions $f(\b{x}) = \b{x}$). The $\b{\mu}_x$ is modeled by the intercepts included as input to the neurons of autoencoder layers. Figure \ref{figure_autoencoder} shows this autoencoder.
As we do not have any non-linearity between the projections, we can define:
\begin{align}
&\ddot{\b{U}} := \b{U}_1 \cdots \b{U}_m \implies \ddot{\b{U}}^\top = \b{U}_m^\top \cdots \b{U}_1^\top, \\
&\therefore ~~~~ \widehat{\b{X}} = \ddot{\b{U}} \ddot{\b{U}}^\top \b{X} + \b{\mu}_x. \label{equation_linear_autoencoder_reconstruction}
\end{align}
The Eq. (\ref{equation_linear_autoencoder_reconstruction}) shows that the whole autoencoder can be reduced to an undercomplete autoencoder with one hidden layer where the weight matrix is $\ddot{\b{U}}$ (see Fig. \ref{figure_autoencoder}).
In other words, in autoencoder neural network, every layer excluding the activation function behaves as a linear projection.

Comparing the Eqs. (\ref{equation_reconstruction_severalPoint}) and (\ref{equation_linear_autoencoder_reconstruction}) shows that the whole autoencoder is reduced to PCA.
Therefore, \textit{PCA is equivalent to an undercomplete autoencoder with one hidden layer without activation function}. Therefore, if we trained weights of such an autoencoder by back-propagation \cite{rumelhart1986learning} are roughly equal to the PCA directions. 
Moreover, as PCA is the best linear projection in terms of reconstruction error, \textit{if we have an undercomplete autoencoder with ``one'' hidden layer, it is best not to use any activation function}; this is not noticed by some papers in the literature, unfortunately.

We saw that an autoencoder with $2m$ hidden layers without activation function reduces to linear PCA.
This explains why in autoencoders with more than one layer, we use non-linear activation function $f(.)$ as:
\begin{align}
&\widehat{\b{X}} = f^{-1}(\b{U}_1 \dots f^{-1}(\b{U}_m f(\b{U}_m^\top \nonumber \\
&~~~~~~~~~~~~~~~~~~~ \dots f(\b{U}_1^\top \b{X})\dots))\dots) + \b{\mu}_x.
\end{align}

\subsection{PCA Using Singular Value Decomposition}

The PCA can be done using Singular Value Decomposition (SVD) of $\breve{\b{X}}$, rather than eigen-decomposition of $\b{S}$. Consider the complete SVD of $\breve{\b{X}}$ (see Appendix \ref{section_appendix_SVD}):
\begin{align}
\mathbb{R}^{d \times n} \ni \breve{\b{X}} = \b{U}\b{\Sigma}\b{V}^\top,
\end{align}
where the columns of $\b{U} \in \mathbb{R}^{d \times d}$ (called left singular vectors) are the eigenvectors of $\breve{\b{X}}\breve{\b{X}}^\top$, the columns of $\b{V} \in \mathbb{R}^{n \times n}$ (called right singular vectors) are the eigenvectors of $\breve{\b{X}}^\top \breve{\b{X}}$, and the $\b{\Sigma} \in \mathbb{R}^{d \times n}$ is a rectangular diagonal matrix whose diagonal entries (called singular values) are the square root of eigenvalues of $\breve{\b{X}}\breve{\b{X}}^\top$ and/or $\breve{\b{X}}^\top \breve{\b{X}}$. 
See Proposition \ref{proposition_SVD} in Appendix \ref{section_appendix_SVD} for proof of this claim.

According to Eq. (\ref{equation_covariance_matrix}), the $\breve{\b{X}}\breve{\b{X}}^\top$ is the covariance matrix $\b{S}$. In Eq. (\ref{equation_pca_eigendecomposition_2}), we saw that the eigenvectors of $\b{S}$ are the principal directions. On the other hand, here, we saw that the columns of $\b{U}$ are the eigenvectors of $\breve{\b{X}}\breve{\b{X}}^\top$.
Hence, we can apply SVD on $\breve{\b{X}}$ and take the left singular vectors (columns of $\b{U}$) as the principal directions.  

An interesting thing is that in SVD of $\breve{\b{X}}$, the columns of $\b{U}$ are automatically sorted from largest to smallest singular values (eigenvalues) and we do not need to sort as we did in using eigenvalue decomposition for the covariance matrix.

\subsection{Determining the Number of Principal Directions}

Usually in PCA, the components with smallest eigenvalues are cut off to reduce the data. There are different methods for estimating the best number of components to keep (denoted by $p$), such as using Bayesian model selection \cite{minka2001automatic}, scree plot \cite{cattell1966scree}, and comparing the ratio $\lambda_j / \sum_{k=1}^d \lambda_k$ with a threshold \cite{abdi2010principal} where $\lambda_i$ denotes the eigenvalue related to the $j$-th principal component.
Here, we explain the two methods of scree plot and the ratio.

The scree plot \cite{cattell1966scree} is a plot of the eigenvalues versus sorted components from the leading (having largest eigenvalue) to trailing (having smallest eigenvalue). A threshold for the vertical (eigenvalue) axis chooses the components with the large enough eigenvalues and removes the rest of the components.
A good threshold is where the eigenvalue drops significantly. In most of the datasets, a significant drop of eigenvalue occurs.

Another way to choose the best components is the ratio \cite{abdi2010principal}:
\begin{align}
\frac{\lambda_j}{\sum_{k=1}^d \lambda_k},
\end{align}
for the $j$-th component. Then, we sort the features from the largest to smallest ratio and select the $p$ best components or up to the component where a significant drop of the ratio happens.

\section{Dual Principal Component Analysis}

Assume the case where the dimensionality of data is high and much greater than the sample size, i.e., $d \gg n$.
In this case, consider the incomplete SVD of $\breve{\b{X}}$ (see Appendix \ref{section_appendix_SVD}):
\begin{align}\label{equation_SVD_dual}
\breve{\b{X}} = \b{U}\b{\Sigma}\b{V}^\top,
\end{align}
where here, $\b{U} \in \mathbb{R}^{d \times p}$ and $\b{V} \in \mathbb{R}^{n \times p}$ contain the $p$ leading left and right singular vectors of $\breve{\b{X}}$, respectively, where $p$ is the number of ``non-zero'' singular values of $\breve{\b{X}}$ and usually $p \ll d$. Here, the $\b{\Sigma} \in \mathbb{R}^{p \times p}$ is a square matrix having the $p$ largest non-zero singular values of $\breve{\b{X}}$. As the $\b{\Sigma}$ is a square diagonal matrix and its diagonal includes non-zero entries (is full-rank), it is invertible \cite{ghodsi2006dimensionality}.
Therefore, $\b{\Sigma}^{-1} = \textbf{diag}([\frac{1}{\sigma_1}, \dots, \frac{1}{\sigma_p}]^\top)$ if we have $\b{\Sigma} = \textbf{diag}([\sigma_1, \dots, \sigma_p]^\top)$.

\subsection{Projection}

Recall Eq. (\ref{equation_projection_severalPoint}) for projection onto PCA subspace: $\widetilde{\b{X}} = \b{U}^\top \breve{\b{X}}$.
On the other hand, according to Eq. (\ref{equation_SVD_dual}), we have:
\begin{align}\label{equation_dual_1}
\breve{\b{X}} = \b{U}\b{\Sigma}\b{V}^\top \implies \b{U}^\top\breve{\b{X}} = \underbrace{\b{U}^\top\b{U}}_{\b{I}}\b{\Sigma}\b{V}^\top = \b{\Sigma}\b{V}^\top.
\end{align}
According to Eqs. (\ref{equation_projection_severalPoint}) and (\ref{equation_dual_1}), we have:
\begin{align}\label{equation_projected_dual}
\therefore ~~~~~~~ \widetilde{\b{X}} = \b{\Sigma}\b{V}^\top
\end{align}
The Eq. (\ref{equation_projected_dual}) can be used for projecting data onto PCA subspace instead of Eq. (\ref{equation_projection_severalPoint}). This is projection of training data in dual PCA.

\subsection{Reconstruction}

According to Eq. (\ref{equation_SVD_dual}), we have:
\begin{align}
\breve{\b{X}} = \b{U}\b{\Sigma}\b{V}^\top &\implies \breve{\b{X}}\b{V} = \b{U}\b{\Sigma}\underbrace{\b{V}^\top\b{V}}_{\b{I}} = \b{U}\b{\Sigma} \nonumber \\
&\implies \b{U} = \breve{\b{X}}\b{V}\b{\Sigma}^{-1}. \label{equation_dual_U}
\end{align}

Plugging Eq. (\ref{equation_dual_U}) in Eq. (\ref{equation_reconstruction_severalPoint}) gives us:
\begin{align}
\widehat{\b{X}} &= \b{U} \widetilde{\b{X}} + \b{\mu}_x \overset{(\ref{equation_dual_U})}{=} \breve{\b{X}}\b{V}\b{\Sigma}^{-1}\widetilde{\b{X}} + \b{\mu}_x \nonumber \\
&\overset{(\ref{equation_projected_dual})}{=} \breve{\b{X}}\b{V}\underbrace{\b{\Sigma}^{-1}\b{\Sigma}}_{\b{I}}\b{V}^\top + \b{\mu}_x \nonumber \\
&\implies \widehat{\b{X}} = \breve{\b{X}}\b{V}\b{V}^\top + \b{\mu}_x. \label{equation_reconstruction_dual}
\end{align}
The Eq. (\ref{equation_reconstruction_dual}) can be used for reconstruction of data instead of Eq. (\ref{equation_reconstruction_severalPoint}).
This is reconstruction of training data in dual PCA.

\subsection{Out-of-sample Projection}

Recall Eq. (\ref{equation_outOfSample_projection_PCA}) for projection of an out-of-sample point $\b{x}_t$ onto PCA subspace.
According to Eq. (\ref{equation_dual_U}), we have:
\begin{align}
&\b{U}^\top \overset{(\ref{equation_dual_U})}{=} \b{\Sigma}^{-\top}\b{V}^\top\breve{\b{X}}^\top \overset{(a)}{=} \b{\Sigma}^{-1}\b{V}^\top\breve{\b{X}}^\top \label{equation_U_transpose_dual}\\
&\overset{(\ref{equation_outOfSample_projection_PCA})}{\implies} \widetilde{\b{x}}_t = \b{\Sigma}^{-1}\b{V}^\top\breve{\b{X}}^\top \breve{\b{x}}_t, \label{equation_outOfSample_projection_dual}
\end{align}
where $(a)$ is because $\b{\Sigma}^{-1}$ is diagonal and thus symmetric.
The Eq. (\ref{equation_outOfSample_projection_dual}) can be used for projecting out-of-sample data point onto PCA subspace instead of Eq. (\ref{equation_outOfSample_projection_PCA}).
This is out-of-sample projection in dual PCA.

Considering all the $n_t$ out-of-sample data points, the projection is:
\begin{align}
\widetilde{\b{X}}_t = \b{\Sigma}^{-1}\b{V}^\top\breve{\b{X}}^\top \breve{\b{X}}_t.
\end{align}

\subsection{Out-of-sample Reconstruction}

Recall Eq. (\ref{equation_outOfSample_reconstruct_PCA}) for reconstruction of an out-of-sample point $\b{x}_t$.
According to Eqs. (\ref{equation_dual_U}) and (\ref{equation_U_transpose_dual}), we have:
\begin{align}
&\b{U}\b{U}^\top = \breve{\b{X}}\b{V}\b{\Sigma}^{-1} \b{\Sigma}^{-1}\b{V}^\top\breve{\b{X}}^\top \nonumber\\
&\overset{(\ref{equation_outOfSample_reconstruct_PCA})}{\implies} \widehat{\b{x}}_t = \breve{\b{X}}\b{V}\b{\Sigma}^{-2}\b{V}^\top\breve{\b{X}}^\top \breve{\b{x}}_t + \b{\mu}_x. \label{equation_outOfSample_reconstruct_dual}
\end{align}
The Eq. (\ref{equation_outOfSample_reconstruct_dual}) can be used for reconstruction of an out-of-sample data point instead of Eq. (\ref{equation_outOfSample_reconstruct_PCA}). This is out-of-sample reconstruction in dual PCA.

Considering all the $n_t$ out-of-sample data points, the reconstruction is:
\begin{align}
\widehat{\b{X}}_t = \breve{\b{X}}\b{V}\b{\Sigma}^{-2}\b{V}^\top\breve{\b{X}}^\top \breve{\b{X}}_t + \b{\mu}_x.
\end{align}

\subsection{Why is Dual PCA Useful?}

The dual PCA can be useful for two reasons:
\begin{enumerate}
\item As can be seen in Eqs. (\ref{equation_projected_dual}), (\ref{equation_reconstruction_dual}), (\ref{equation_outOfSample_projection_dual}), and (\ref{equation_outOfSample_reconstruct_dual}), the formulae for dual PCA only include $\b{V}$ and not $\b{U}$. The columns of $\b{V}$ are the eigenvectors of $\breve{\b{X}}^\top \breve{\b{X}} \in \mathbb{R}^{n \times n}$ and the columns of $\b{U}$ are the eigenvectors of $\breve{\b{X}} \breve{\b{X}}^\top \in \mathbb{R}^{d \times d}$. In case the dimensionality of data is much high and greater than the sample size, i.e., $n \ll d$, computation of eigenvectors of $\breve{\b{X}}^\top \breve{\b{X}}$ is easier and faster than $\breve{\b{X}} \breve{\b{X}}^\top$ and also requires less storage. Therefore, dual PCA is more efficient than direct PCA in this case in terms of both speed and storage. Note that the results of PCA and dual PCA are exactly the same.
\item Some inner product forms, such as $\breve{\b{X}}^\top \breve{\b{x}}_t$, have appeared in the formulae of dual PCA. This provides opportunity for kernelizing the PCA to have kernel PCA using the so-called kernel trick. As will be seen in the next section, we use dual PCA in formulation of kernel PCA.
\end{enumerate}

\section{Kernel Principal Component Analysis}

\begin{figure}[!t]
\centering
\includegraphics[width=3.2in]{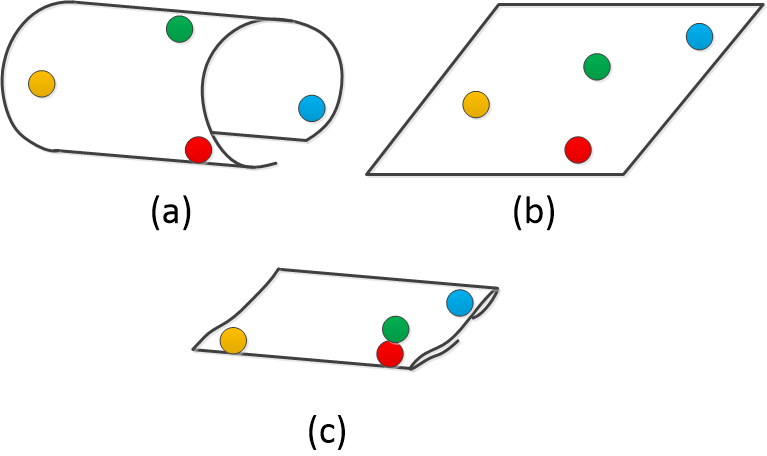}
\caption{(a) A 2D nonlinear manifold where the data exist on in the 3D original space. As the manifold is nonlinear, the geodesic distances of points on the manifold are different from their Euclidean distances. (b) The correct unfolded manifold where the geodesic distances of points on the manifold have been preserved. (c) Applying the linear PCA, which takes Euclidean distances into account, on the nonlinear data where the found subspace has ruined the manifold so the far away red and green points have fallen next to each other. The credit of this example is for Prof. Ali Ghodsi.}
\label{figure_nonlinear_manifold}
\end{figure}

The PCA is a linear method because the projection is linear. In case the data points exist on a non-linear sub-manifold, the linear subspace learning might not be completely effective. For example, see Fig. \ref{figure_nonlinear_manifold}.

In order to handle this problem of PCA, we have two options. We should either change PCA to become a nonlinear method or we can leave the PCA to be linear but change the data hoping to fall on a linear or close to linear manifold. Here, we do the latter so we change the data. We increase the dimensionality of data by mapping the data to feature space with higher dimensionality hoping that in the feature space, it falls on a linear manifold. This is referred to as ``blessing of dimensionality'' in the literature \cite{donoho2000high} which is pursued using kernels \cite{hofmann2008kernel}. This PCA method which uses the kernel of data is named ``kernel PCA'' \cite{scholkopf1997kernel}.

\subsection{Kernels and Hilbert Space}

Suppose that $\b{\phi}: \b{x} \rightarrow \mathcal{H}$ is a function which maps the data $\b{x}$ to Hilbert space (feature space). The $\b{\phi}$ is called ``pulling function''. In other words, $\b{x} \mapsto \b{\phi}(\b{x})$. Let $t$ denote the dimensionality of the feature space, i.e., $\b{\phi}(\b{x}) \in \mathbb{R}^t$ while $\b{x} \in \mathbb{R}^d$. Note that we usually have $t \gg d$.

If $\mathcal{X}$ denotes the set of points, i.e., $\b{x} \in \mathcal{X}$, the kernel of two vectors $\b{x}_1$ and $\b{x}_2$ is $k: \mathcal{X} \times \mathcal{X} \rightarrow \mathbb{R}$ and is defined as \cite{hofmann2008kernel,herbrich2001learning}:
\begin{align}
k(\b{x}_1, \b{x}_2) := \b{\phi}(\b{x}_1)^\top \b{\phi}(\b{x}_2),
\end{align}
which is a measure of ``similarity'' between the two vectors because the inner product captures similarity.

We can compute the kernel of two matrices $\b{X}_1 \in \mathbb{R}^{d \times n_1}$ and $\b{X}_2 \in \mathbb{R}^{d \times n_2}$ and have a ``kernel matrix'' (also called ``Gram matrix''):
\begin{align}
\mathbb{R}^{n_1 \times n_2} \ni \b{K}(\b{X}_1, \b{X}_2) := \b{\Phi}(\b{X}_1)^\top \b{\Phi}(\b{X}_2),
\end{align}
where $\b{\Phi}(\b{X}_1) := [\b{\phi}(\b{x}_1), \dots, \b{\phi}(\b{x}_n)] \in \mathbb{R}^{t \times n_1}$ is the matrix of mapped $\b{X}_1$ to the feature space. The $\b{\Phi}(\b{X}_2) \in \mathbb{R}^{t \times n_2}$ is defined similarly. 
We can compute the kernel matrix of dataset $\b{X} \in \mathbb{R}^{d \times n}$ over itself:
\begin{align}\label{equation_kernel_matrix_of_X}
\mathbb{R}^{n \times n} \ni \b{K}_x := \b{K}(\b{X}, \b{X}) = \b{\Phi}(\b{X})^\top \b{\Phi}(\b{X}),
\end{align}
where $\b{\Phi}(\b{X}) := [\b{\phi}(\b{x}_1), \dots, \b{\phi}(\b{x}_n)] \in \mathbb{R}^{t \times n}$ is the pulled (mapped) data.

Note that in kernel methods, the pulled data $\b{\Phi}(\b{X})$ are usually not available and merely the kernel matrix $\b{K}(\b{X}, \b{X})$, which is the inner product of the pulled data with itself, is available.

There exist different types of kernels. Some of the most well-known kernels are:
\begin{align}
&\text{Linear:} ~~ k(\b{x}_1, \b{x}_2) = \b{x}_1^\top \b{x}_2 + c_1, \\
&\text{Polynomial:} ~~ k(\b{x}_1, \b{x}_2) = (c_1\b{x}_1^\top \b{x}_2 + c_2)^{c_3}, \\
&\text{Gaussian:} ~~ k(\b{x}_1, \b{x}_2) = \exp\big(\!-\frac{||\b{x}_1 - \b{x}_2||_2^2}{2\sigma^2}\big), \\
&\text{Sigmoid:} ~~ k(\b{x}_1, \b{x}_2) = \tanh(c_1\b{x}_1^\top\b{x}_2 + c_2), 
\end{align}
where $c_1$, $c_2$, $c_3$, and $\sigma$ are scalar constants. The Gaussian and Sigmoid kernels are also called Radial Basis Function (RBF) and hyperbolic tangent, respectively. Note that the Gaussian kernel can also be written as $\exp\big(\!-\gamma||\b{x}_1 - \b{x}_2||_2^2\big)$ where $\gamma > 0$.

It is noteworthy to mention that in the RBF kernel, the dimensionality of the feature space is infinite. The reason lies in the Maclaurin series expansion (Taylor series expansion at zero) of this kernel:
\begin{align*}
\exp(-\gamma r) \approx 1 - \gamma r + \frac{\gamma^2}{2!} r^2 - \frac{\gamma^3}{3!} r^3 + \dots,
\end{align*}
where $r := ||\b{x}_1 - \b{x}_2||_2^2$, which is infinite dimensional with respect to $r$.

It is also worth mentioning that if we want the pulled data $\b{\Phi}(\b{X})$ to be centered, i.e.:
\begin{align}
\breve{\b{\Phi}}(\b{X}) := \b{\Phi}(\b{X})\b{H},
\end{align}
we should double center the kernel matrix (see Appendix \ref{section_appendix_centering}) because if we use centered pulled data in Eq. (\ref{equation_kernel_matrix_of_X}), we have:
\begin{align*}
&\breve{\b{\Phi}}(\b{X})^\top \breve{\b{\Phi}}(\b{X}) = \big(\b{\Phi}(\b{X})\b{H}\big)^\top \big(\b{\Phi}(\b{X})\b{H}\big) \\
&\overset{(\ref{equation_centeringMatrix_is_symmetric})}{=} \b{H} \b{\Phi}(\b{X})^\top \b{\Phi}(\b{X})\b{H} \overset{(\ref{equation_kernel_matrix_of_X})}{=} \b{H} \b{K}_x \b{H}, 
\end{align*}
which is the double-centered kernel matrix.
Thus:
\begin{align}\label{equation_doubleCentered_kernel}
\breve{\b{K}}_x := \b{H} \b{K}_x \b{H} = \breve{\b{\Phi}}(\b{X})^\top \breve{\b{\Phi}}(\b{X}),
\end{align}
where $\breve{\b{K}}_x$ denotes the double-centered kernel matrix (see Appendix \ref{section_appendix_centeringKernel}).

\subsection{Projection}

We apply incomplete SVD on the centered pulled (mapped) data $\breve{\b{\Phi}}(\b{X})$ (see Appendix \ref{section_appendix_SVD}):
\begin{align}\label{equation_kernelPCA_SVD_Phi}
\mathbb{R}^{t \times n} \ni \breve{\b{\Phi}}(\b{X}) = \b{U}\b{\Sigma}\b{V}^\top,
\end{align}
where $\b{U} \in \mathbb{R}^{t \times p}$ and $\b{V} \in \mathbb{R}^{n \times p}$ contain the $p$ leading left and right singular vectors of $\breve{\b{\Phi}}(\b{X})$, respectively, where $p$ is the number of ``non-zero'' singular values of $\breve{\b{\Phi}}(\b{X})$ and usually $p \ll t$. Here, the $\b{\Sigma} \in \mathbb{R}^{p \times p}$ is a square matrix having the $p$ largest non-zero singular values of $\breve{\b{\Phi}}(\b{X})$. 

However, as mentioned before, the pulled data are not necessarily available so Eq. (\ref{equation_kernelPCA_SVD_Phi}) cannot be done.
The kernel, however, is available. Therefore, we apply eigen-decomposition \cite{ghojogh2019eigenvalue} to the double-centered kernel:
\begin{align}\label{equation_kernelPCA_eigen_centered_kernel}
\breve{\b{K}}_x \b{V} = \b{V} \b{\Lambda},
\end{align}
where the columns of $\b{V}$ and the diagonal of $\b{\Lambda}$ are the eigenvectors and eigenvalues of $\breve{\b{K}}_x$, respectively.
The columns of $\b{V}$ in Eq. (\ref{equation_kernelPCA_SVD_Phi}) are the right singular vectors of $\breve{\b{\Phi}}(\b{X})$ which are equivalent to the eigenvectors of $\breve{\b{\Phi}}(\b{X})^\top \breve{\b{\Phi}}(\b{X}) = \breve{\b{K}}_x$,
according to Proposition \ref{proposition_SVD} in Appendix \ref{section_appendix_SVD}.
Also, according to that proposition, the diagonal of $\b{\Sigma}$ in Eq. (\ref{equation_kernelPCA_SVD_Phi}) is equivalent to the square root of eigenvalues of $\breve{\b{K}}_x$. 

Therefore, in practice where the pulling function is not necessarily available, we use Eq. (\ref{equation_kernelPCA_eigen_centered_kernel}) in order to find the $\b{V}$ and $\b{\Sigma}$ in Eq. (\ref{equation_kernelPCA_SVD_Phi}). The Eq. (\ref{equation_kernelPCA_eigen_centered_kernel}) can be restated as:
\begin{align}\label{equation_kernelPCA_eigen_centered_kernel_2}
\breve{\b{K}}_x \b{V} = \b{V} \b{\Sigma}^2,
\end{align}
to be compatible to Eq. (\ref{equation_kernelPCA_SVD_Phi}).
It is noteworthy that because of using Eq. (\ref{equation_kernelPCA_eigen_centered_kernel_2}) instead of Eq. (\ref{equation_kernelPCA_SVD_Phi}), \textit{the projection directions $\b{U}$ are not available in kernel PCA to be observed or plotted.}

Similar to what we did for Eq. (\ref{equation_projected_dual}):
\begin{align}
&\breve{\b{\Phi}}(\b{X}) = \b{U}\b{\Sigma}\b{V}^\top \nonumber \\
&\implies \b{U}^\top\breve{\b{\Phi}}(\b{X}) = \underbrace{\b{U}^\top\b{U}}_{\b{I}}\b{\Sigma}\b{V}^\top = \b{\Sigma}\b{V}^\top \nonumber \\
&\therefore ~~~~~~~ \b{\Phi}(\widetilde{\b{X}}) = \b{U}^\top\breve{\b{\Phi}}(\b{X}) = \b{\Sigma}\b{V}^\top, \label{equation_projected_kernel_PCA}
\end{align}
where $\b{\Sigma}$ and $\b{V}$ are obtained from Eq. (\ref{equation_kernelPCA_eigen_centered_kernel_2}).
The Eq. (\ref{equation_projected_kernel_PCA}) is projection of the training data in kernel PCA.

\subsection{Reconstruction}

Similar to what we did for Eq. (\ref{equation_reconstruction_dual}):
\begin{align}
\breve{\b{\Phi}}(\b{X}) = \b{U}\b{\Sigma}\b{V}^\top &\implies \breve{\b{\Phi}}(\b{X})\b{V} = \b{U}\b{\Sigma}\underbrace{\b{V}^\top\b{V}}_{\b{I}} = \b{U}\b{\Sigma} \nonumber \\
&\implies \b{U} = \breve{\b{\Phi}}(\b{X})\b{V}\b{\Sigma}^{-1}. \label{equation_kernel_U} 
\end{align}
Therefore, the reconstruction is:
\begin{alignat}{2}
&\b{\Phi}(\widehat{\b{X}}) &&= \b{U}\b{\Phi}(\widetilde{\b{X}}) + \b{\mu}_x \overset{(\ref{equation_kernel_U})}{=} \breve{\b{\Phi}}(\b{X})\b{V}\b{\Sigma}^{-1}\b{\Phi}(\widetilde{\b{X}}) + \b{\mu}_x \nonumber \\ 
& &&\overset{(\ref{equation_projected_kernel_PCA})}{=} \breve{\b{\Phi}}(\b{X})\b{V}\underbrace{\b{\Sigma}^{-1}\b{\Sigma}}_{\b{I}}\b{V}^\top + \b{\mu}_x \nonumber \\
&\implies &&\b{\Phi}(\widehat{\b{X}}) = \breve{\b{\Phi}}(\b{X})\b{V}\b{V}^\top + \b{\mu}_x. \label{equation_reconstruction_kernel_PCA}
\end{alignat}
However, the $\breve{\b{\Phi}}(\b{X})$ is not available necessarily; therefore, we cannot reconstruct the training data in kernel PCA.

\subsection{Out-of-sample Projection}

Similar to what we did for Eq. (\ref{equation_outOfSample_projection_dual}):
\begin{align}
&\b{U}^\top \overset{(\ref{equation_kernel_U})}{=} \b{\Sigma}^{-\top}\b{V}^\top\breve{\b{\Phi}}(\b{X})^\top \overset{(a)}{=} \b{\Sigma}^{-1}\b{V}^\top\breve{\b{\Phi}}(\b{X})^\top \nonumber \\
&\implies \b{\phi}(\widetilde{\b{x}}_t) = \b{U}^\top \breve{\b{\phi}}(\b{x}_t) = \b{\Sigma}^{-1}\b{V}^\top\breve{\b{\Phi}}(\b{X})^\top \breve{\b{\phi}}(\b{x}_t), \nonumber \\
&\overset{(\ref{equation_appendix_centered_kernelVector_outOfSample})}{\implies} \b{\phi}(\widetilde{\b{x}}_t) = \b{\Sigma}^{-1}\b{V}^\top\breve{\b{k}}_t, \label{equation_outOfSample_projection_kernel_PCA}
\end{align}
where $(a)$ is because $\b{\Sigma}^{-1}$ is diagonal and thus symmetric and the $\breve{\b{k}}_t \in \mathbb{R}^n$ is calculated by Eq. (\ref{equation_appendix_doubleCentered_outOfSample_kernel_oneSample}) in Appendix \ref{section_appendix_centeringKernel}.


The Eq. (\ref{equation_outOfSample_projection_kernel_PCA}) is the projection of out-of-sample data in kernel PCA.

Considering all the $n_t$ out-of-sample data points, $\b{X}_t$, the projection is:
\begin{align}
\b{\phi}(\widetilde{\b{X}}_t) = \b{\Sigma}^{-1}\b{V}^\top\breve{\b{K}}_t,
\end{align}
where $\breve{\b{K}}_t$ is calculated by Eq. (\ref{equation_appendix_doubleCentered_outOfSample_kernel}).

\subsection{Out-of-sample Reconstruction}

Similar to what we did for Eq. (\ref{equation_outOfSample_reconstruct_dual}):
\begin{align}
&\implies \b{U}\b{U}^\top \overset{(\ref{equation_kernel_U})}{=} \breve{\b{\Phi}}(\b{X})\b{V}\b{\Sigma}^{-1} \b{\Sigma}^{-1}\b{V}^\top\breve{\b{\Phi}}(\b{X})^\top \nonumber\\
&\implies \b{\phi}(\widehat{\b{x}}_t) = \breve{\b{\Phi}}(\b{X})\b{V}\b{\Sigma}^{-2}\b{V}^\top\breve{\b{\Phi}}(\b{X})^\top \breve{\b{\phi}}(\b{x}_t) + \b{\mu}_x \nonumber \\
&\overset{(\ref{equation_appendix_centered_kernelVector_outOfSample})}{\implies} \b{\phi}(\widehat{\b{x}}_t) = \breve{\b{\Phi}}(\b{X})\b{V}\b{\Sigma}^{-2}\b{V}^\top \breve{\b{k}}_t + \b{\mu}_x, \label{equation_outOfSample_reconstruct_kernel_PCA}
\end{align}
where the $\breve{\b{k}}_t \in \mathbb{R}^n$ is calculated by Eq. (\ref{equation_appendix_doubleCentered_outOfSample_kernel_oneSample}) in Appendix \ref{section_appendix_centeringKernel}.

Considering all the $n_t$ out-of-sample data points, $\b{X}_t$, the reconstruction is:
\begin{align}
\b{\Phi}(\widehat{\b{X}}_t) = \breve{\b{\Phi}}(\b{X})\b{V}\b{\Sigma}^{-2}\b{V}^\top \breve{\b{K}}_t + \b{\mu}_x,
\end{align}
where $\breve{\b{K}}_t$ is calculated by Eq. (\ref{equation_appendix_doubleCentered_outOfSample_kernel}).

In Eq. (\ref{equation_outOfSample_reconstruct_kernel_PCA}), the $\breve{\b{\Phi}}(\b{X})$ appeared at the left of expression, is not available necessarily; therefore, we cannot reconstruct an out-of-sample point in kernel PCA.
According to Eqs. (\ref{equation_reconstruction_kernel_PCA}) and (\ref{equation_outOfSample_reconstruct_kernel_PCA}), we conclude that kernel PCA is not able to reconstruct any data, whether training or out-of-sample.

\subsection{Why is Kernel PCA Useful?}

Finally, it is noteworthy that as the choice of the best kernel might be hard, the kernel PCA is not ``always'' effective in practice \cite{ghodsi2006dimensionality}. 
However, it provides us some useful theoretical insights for explaining the PCA, Multi-Dimensional Scaling (MDS) \cite{cox2008multidimensional}, Isomap \cite{tenenbaum2000global}, Locally Linear Embedding (LLE) \cite{roweis2000nonlinear}, and Laplacian Eigenmap (LE) \cite{belkin2003laplacian} as special cases of kernel PCA with their own kernels (see \cite{ham2004kernel} and chapter 2 in \cite{strange2014open}).

\section{Supervised Principal Component Analysis Using Scoring}

The older version of SPCA used scoring \cite{bair2006prediction}. 
In this version of SPCA, PCA is not a special case of SPCA. The version of SPCA, which will be introduced in the next section, is more solid in terms of theory where PCA is a special case of SPCA.

In SPCA using scoring, we compute the similarity of every feature of data with the class labels and then sort the features and remove the features having the least similarity with the labels. The larger the similarity of a feature with the labels, the better that feature is for discrimination in the embedded subspace. 

Consider the training dataset $\mathbb{R}^{d \times n} \ni \b{X} = [\b{x}_1, \dots, \b{x}_n] = [\b{x}^1, \dots, \b{x}^d]^\top$ where $\b{x}_i \in \mathbb{R}^d$ and $\b{x}^j \in \mathbb{R}^n$ are the $i$-th data point and the $j$-th feature, respectively. 
This type of SPCA is only for classification task so we can consider the dimensionality of the labels to be one, $\ell=1$. Thus, we have $\b{Y} \in \mathbb{R}^{1 \times n}$. We define $\mathbb{R}^n \ni \b{y} := \b{Y}^\top$.

The score of the $j$-th feature, $\b{x}^j$, is:
\begin{align}
\mathbb{R} \ni s_j := \frac{(\b{x}^j)^\top \b{y}}{||(\b{x}^j)^\top \b{x}^j||_2} = \frac{(\b{x}^j)^\top \b{y}}{\sqrt{(\b{x}^j)^\top \b{x}^j}},
\end{align}
After computing the scores of all the features, we sort the features from largest to smallest score. Let $\b{X}' \in \mathbb{R}^{d \times n}$ denote the training dataset whose features are sorted.
We take the $q \leq d$ features with largest scores and remove the other features. Let:
\begin{align}
\mathbb{R}^{q \times n} \ni \b{X}'' := \b{X}'(1:q, :),
\end{align}
be the training dataset with $q$ best features. 

Then, we apply PCA on the $\b{X}'' \in \mathbb{R}^{q \times n}$ rather than $\b{X} \in \mathbb{R}^{d \times n}$. Applying PCA and kernel PCA on $\b{X}''$ results in SPCA and kernel PCA, respectively. 
This type of SPCA was mostly used and popular in bioinformatics for genome data analysis \cite{ma2011principal}.

\section{Supervised Principal Component Analysis Using HSIC}

\subsection{Hilbert-Schmidt Independence Criterion}

Suppose we want to measure the dependence of two random variables. Measuring the correlation between them is easier because correlation is just ``linear'' dependence. 

According to \cite{hein2004kernels}, two random variables are independent if and only if any bounded continuous functions of them are uncorrelated. Therefore, if we map the two random variables $\b{x}$ and $\b{y}$ to two different (``separable'') Reproducing Kernel Hilbert Spaces (RKHSs) and have $\b{\phi}(\b{x})$ and $\b{\phi}(\b{y})$, we can measure the correlation of $\b{\phi}(\b{x})$ and $\b{\phi}(\b{y})$ in Hilbert space to have an estimation of dependence of $\b{x}$ and $\b{y}$ in the original space. 

The correlation of $\b{\phi}(\b{x})$ and $\b{\phi}(\b{y})$ can be computed by the Hilbert-Schmidt norm of the cross-covariance of them \cite{gretton2005measuring}. Note that the squared Hilbert-Schmidt norm of a matrix $\b{A}$ is \cite{bell2016trace}:
\begin{align*}
||\b{A}||_{HS}^2 := \textbf{tr}(\b{A}^\top \b{A}),
\end{align*}
and the cross-covariance matrix of two vectors $\b{x}$ and $\b{y}$ is \cite{gubner2006probability,gretton2005measuring}:
\begin{align*}
\mathbb{C}\text{ov}(\b{x}, \b{y}) := \mathbb{E}\Big[&\big(\b{x} - \mathbb{E}(\b{x})\big) \big(\b{y} - \mathbb{E}(\b{y})\big) \Big].
\end{align*}
Using the explained intuition, an empirical estimation of the Hilbert-Schmidt Independence Criterion (HSIC) is introduced \cite{gretton2005measuring}:
\begin{align}\label{equation_HSIC}
\text{HSIC} := \frac{1}{(n-1)^2}\, \textbf{tr}(\ddot{\b{K}}_x\b{H}\b{K}_y\b{H}),
\end{align}
where $\ddot{\b{K}}_x$ and $\b{K}_y$ are the kernels over $\b{x}$ and $\b{y}$, respectively. In other words, $\ddot{\b{K}}_x = \b{\phi}(\b{x})^\top \b{\phi}(\b{x})$ and $\b{K}_y = \b{\phi}(\b{y})^\top \b{\phi}(\b{y})$. 
We are using $\ddot{\b{K}}_x$ rather than $\b{K}_x$ because $\b{K}_x$ is going to be used in kernel SPCA in the next sections. 
The term $1/(n-1)^2$ is used for normalization.
The $\b{H}$ is the centering matrix (see Appendix \ref{section_appendix_centering}):
\begin{align}
\mathbb{R}^{n \times n} \ni \b{H} = \b{I} - (1/n) \b{1}\b{1}^\top.
\end{align}
The $\b{H}\b{K}_y\b{H}$ double centers the $\b{K}_y$ in HSIC.

The HSIC (Eq. (\ref{equation_HSIC})) measures the dependence of two random variable vectors $\b{x}$ and $\b{y}$. Note that $\text{HSIC}=0$ and $\text{HSIC}>0$ mean that $\b{x}$ and $\b{y}$ are independent and dependent, respectively. The greater the HSIC, the greater dependence they have.

\subsection{Supervised PCA}

Supervised PCA (SPCA) \cite{barshan2011supervised} uses the HSIC. We have the data $\b{X} = [\b{x}_1, \dots, \b{x}_n] \in \mathbb{R}^{d \times n}$ and the labels $\b{Y} = [\b{y}_1, \dots, \b{y}_n] \in \mathbb{R}^{\ell \times n}$, where $\ell$ is the dimensionality of the labels and we usually have $\ell=1$. However, in case the labels are encoded (e.g., one-hot-encoded) or SPCA is used for regression (e.g., see \cite{ghojogh2019instance}), we have $\ell > 1$.

SPCA tries to maximize the dependence of the projected data points $\b{U}^\top \b{X}$ and the labels $\b{Y}$. It uses a linear kernel for the projected data points: 
\begin{align}\label{equation_SPCA_kernel_of_projection}
\ddot{\b{K}}_x = (\b{U}^\top \b{X})^\top (\b{U}^\top \b{X}) = \b{X}^\top \b{U} \b{U}^\top \b{X},
\end{align}
and an arbitrary kernel $\b{K}_y$ over $\b{Y}$.
For classification task, one of the best choices for the $\b{K}_y$ is delta kernel \cite{barshan2011supervised} where the $(i,j)$-th element of kernel is:
\begin{align}
\b{K}_y = \delta_{\b{y}_i, \b{y}_j} := 
\left\{
\begin{array}{ll}
  1 & \text{if } \b{y}_i = \b{y}_j,\\
  0 & \text{if } \b{y}_i \neq \b{y}_j,
\end{array}
\right.
\end{align}
where $\delta_{\b{y}_i, \b{y}_j}$ is the Kronecker delta which is one if the $\b{x}_i$ and $\b{x}_j$ belong to the same class.

Another good choice for kernel in classification task in SPCA is an arbitrary kernel (e.g., linear kernel $\b{K}_y = \b{Y}^\top \b{Y}$) over $\b{Y}$ where the columns of $\b{Y}$ are one-hot encoded. This is a good choice because the distances of classes will be equal; otherwise, some classes will fall closer than the others for no reason and fairness between classes goes away.

The SPCA can also be used for regression (e.g., see \cite{ghojogh2019instance}) and that is one of the advantages of SPCA. In that case, a good choice for $\b{K}_y$ is an arbitrary kernel (e.g., linear kernel $\b{K}_y = \b{Y}^\top \b{Y}$) over $\b{Y}$ where the columns of the $\b{Y}$, i.e., labels, are the observations in regression. Here, the distances of observations have meaning and should not be manipulated.

The HSIC in SPCA case becomes: 
\begin{align}\label{equation_HSIC_SPCA}
\text{HSIC} = \frac{1}{(n-1)^2}\, \textbf{tr}(\b{X}^\top \b{U} \b{U}^\top \b{X}\b{H}\b{K}_y\b{H}).
\end{align}
where $\b{U} \in \mathbb{R}^{d \times p}$ is the unknown projection matrix for projection onto the SPCA subspace and should be found. The desired dimensionality of the subspace is $p$ and usually $p \ll d$. 

We should maximize the HSIC in order to maxzimize the dependence of $\b{U}^\top \b{X}$ and $\b{Y}$. Hence:
\begin{equation}
\begin{aligned}
& \underset{\b{U}}{\text{maximize}}
& & \textbf{tr}(\b{X}^\top \b{U} \b{U}^\top \b{X}\b{H}\b{K}_y\b{H}), \\
& \text{subject to}
& & \b{U}^\top \b{U} = \b{I},
\end{aligned}
\end{equation}
where the constraint ensures that the $\b{U}$ is an orthogonal matrix, i.e., the SPCA directions are orthonormal.

Using Lagrangian \cite{boyd2004convex}, we have:
\begin{align*}
\mathcal{L} &= \textbf{tr}(\b{X}^\top \b{U} \b{U}^\top \b{X}\b{H}\b{K}_y\b{H}) - \textbf{tr}\big(\b{\Lambda}^\top (\b{U}^\top \b{U} - \b{I})\big) \\
&\overset{(a)}{=} \textbf{tr}(\b{U} \b{U}^\top \b{X}\b{H}\b{K}_y\b{H}\b{X}^\top) - \textbf{tr}\big(\b{\Lambda}^\top (\b{U}^\top \b{U} - \b{I})\big),
\end{align*}
where $(a)$ is because of the cyclic property of trace and $\b{\Lambda} \in \mathbb{R}^{p \times p}$ is a diagonal matrix $\textbf{diag}([\lambda_1, \dots, \lambda_p]^\top)$ including the Lagrange multipliers.
Setting the derivative of Lagrangian to zero gives:
\begin{align}
& \mathbb{R}^{d \times p} \ni \frac{\partial \mathcal{L}}{\partial \b{U}} = 2 \b{X}\b{H}\b{K}_y\b{H}\b{X}^\top \b{U} - 2\b{U} \b{\Lambda} \overset{\text{set}}{=} 0 \nonumber \\ 
& \implies \b{X}\b{H}\b{K}_y\b{H}\b{X}^\top \b{U} = \b{U} \b{\Lambda}, \label{eqaution_SPCA_eigendecomposition}
\end{align}
which is the eigen-decomposition of $\b{X}\b{H}\b{K}_y\b{H}\b{X}^\top$ where the columns of $\b{U}$ and the diagonal of $\b{\Lambda}$ are the eigenvectors and eigenvalues of $\b{X}\b{H}\b{K}_y\b{H}\b{X}^\top$, respectively \cite{ghojogh2019eigenvalue}. The eigenvectors and eigenvalues are sorted from the leading (largest eigenvalue) to the trailing (smallest eigenvalue) because we are maximizing in the optimization problem.
As a conclusion, if projecting onto the SPCA subspace or $\textbf{span}\{\b{u}_1, \dots, \b{u}_p\}$, the SPCA directions $\{\b{u}_1, \dots, \b{u}_p\}$ are the sorted eigenvectors of $\b{X}\b{H}\b{K}_y\b{H}\b{X}^\top$.
In other words, the columns of the projection matrix $\b{U}$ in SPCA are the $p$ leading eigenvectors of $\b{X}\b{H}\b{K}_y\b{H}\b{X}^\top$.

Similar to what we had in PCA, the projection, projection of out-of-sample, reconstruction, and reconstruction of out-of-sample in SPCA are: 
\begin{align}
&\widetilde{\b{X}} = \b{U}^\top \b{X}, \label{equation_SPCA_projection} \\
&\widetilde{\b{x}}_t = \b{U}^\top \b{x}_t, \label{equation_SPCA_projection_outOfSample} \\
&\widehat{\b{X}} = \b{U} \b{U}^\top \b{X} = \b{U} \widetilde{\b{X}}, \label{equation_SPCA_reconstruction} \\
&\widehat{\b{x}}_t = \b{U} \b{U}^\top \b{x}_t = \b{U} \widetilde{\b{x}}_t, \label{equation_SPCA_reconstruction_outOfSample}
\end{align}
respectively.
In SPCA, there is no need to center the data as the centering is already handled by $\b{H}$ in HSIC. This gets more clear in the following section where we see that PCA is a special case of SPCA.
Note that in the equations of SPCA, although not necessary, we can center the data and in that case, the mean of embedding in the subspace will be zero.

Considering all the $n_t$ out-of-sample data points, the projection and reconstruction are:
\begin{align}
&\widetilde{\b{X}}_t = \b{U}^\top \b{X}_t, \\
&\widehat{\b{X}}_t = \b{U} \b{U}^\top \b{X}_t = \b{U} \widetilde{\b{X}}_t,
\end{align}
respectively.

\subsection{PCA is a special case of SPCA!}

Not considering the similarities of the labels means that we do not care about the class labels so we are unsupervised. 
if we do not consider the similarities of labels, the kernel over the labels becomes the identity matrix, $\b{K}_y = \b{I}$.
According to Eq. (\ref{eqaution_SPCA_eigendecomposition}), SPCA is the eigen-decomposition of $\b{X}\b{H}\b{K}_y\b{H}\b{X}^\top$. In this case, this matrix becomes:
\begin{align*}
\b{X}\b{H}\b{K}_y\b{H}\b{X}^\top &= \b{X}\b{H}\b{I}\b{H}\b{X}^\top \overset{(\ref{equation_centeringMatrix_is_symmetric})}{=} \b{X}\b{H}\b{I}\b{H}^\top\b{X}^\top \\
&= \b{X}\b{H}\b{H}^\top\b{X}^\top = (\b{X}\b{H})(\b{X}\b{H})^\top \\
&\overset{(\ref{equation_centered_training_data})}{=} \breve{\b{X}} \breve{\b{X}}^\top \overset{(\ref{equation_covariance_matrix})}{=} \b{S},
\end{align*}
which is the covariance matrix whose eigenvectors are the PCA directions.
Thus, if we do not consider the similarities of labels, i.e., we are unsupervised, SPCA reduces to PCA as expected.

\subsection{Dual Supervised PCA}

The SPCA can be formulated in dual form \cite{barshan2011supervised}. 
We saw that in SPCA, the columns of $\b{U}$ are the eigenvectors of $\b{X}\b{H}\b{K}_y\b{H}\b{X}^\top$. We apply SVD on $\b{K}_y$ (see Appendix \ref{section_appendix_SVD}):
\begin{align*}
\mathbb{R}^{n \times n} \ni \b{K}_y = \b{Q} \b{\Omega} \b{Q}^\top, 
\end{align*}
where $\b{Q} \in \mathbb{R}^{n \times n}$ includes left or right singular vectors and $\b{\Omega} \in \mathbb{R}^{n \times n}$ contains the singular values of $\b{K}_y$. Note that the left and right singular vectors are equal because $\b{K}_y$ is symmetric and thus $\b{K}_y \b{K}_y^\top$ and $\b{K}_y^\top \b{K}_y$ are equal. As $\b{\Omega}$ is a diagonal matrix with non-negative entries, we can decompose it to $\b{\Omega} = \b{\Omega}^{1/2} \b{\Omega}^{1/2} = \b{\Omega}^{1/2} (\b{\Omega}^{1/2})^\top$ where the diagonal entries of $\b{\Omega}^{1/2} \in \mathbb{R}^{n \times n}$ are square root of diagonal entries of $\b{\Omega}$. Therefore, we can decompose $\b{K}_y$ into:
\begin{align}
\b{K}_y &= \b{Q} \b{\Omega}^{1/2} (\b{\Omega}^{1/2})^\top \b{Q}^\top \nonumber \\
&= (\b{Q} \b{\Omega}^{1/2}) (\b{Q} \b{\Omega}^{1/2})^\top = \b{\Delta} \b{\Delta}^\top, \label{equation_SPCA_decompose_kernelOfLabels}
\end{align}
where:
\begin{align}
\mathbb{R}^{n \times n} \ni \b{\Delta} := \b{Q} \b{\Omega}^{1/2}.
\end{align}
Therefore, we have:
\begin{align*}
\therefore ~~~~ \b{X}\b{H}\b{K}_y\b{H}\b{X}^\top &\overset{(\ref{equation_SPCA_decompose_kernelOfLabels})}{=} \b{X}\b{H}\b{\Delta} \b{\Delta}^\top\b{H}\b{X}^\top \\
&\overset{(\ref{equation_centeringMatrix_is_symmetric})}{=} (\b{X}\b{H}\b{\Delta}) (\b{X}\b{H}\b{\Delta})^\top = \b{\Psi}\b{\Psi}^\top,
\end{align*}
where:
\begin{align}\label{equation_Psi_dual_SPCA}
\mathbb{R}^{d \times n} \ni \b{\Psi} := \b{X}\b{H}\b{\Delta}.
\end{align}

We apply incomplete SVD on $\b{\Psi}$ (see Appendix \ref{section_appendix_SVD}):
\begin{align}\label{equation_Psi_SVD_dual_SPCA}
\mathbb{R}^{d \times n} \ni \b{\Psi} = \b{U} \b{\Sigma} \b{V}^\top, 
\end{align}
where $\b{U} \in \mathbb{R}^{d \times p}$ and $\b{V} \in \mathbb{R}^{d \times p}$ include the $p$ leading left or right singular vectors of $\b{\Psi}$, respectively, and $\b{\Sigma} \in \mathbb{R}^{p \times p}$ contains the $p$ largest singular values of $\b{\Psi}$.

We can compute $\b{U}$ as:
\begin{align}
\b{\Psi} = \b{U} \b{\Sigma} \b{V}^\top &\implies \b{\Psi}\b{V} = \b{U} \b{\Sigma} \underbrace{\b{V}^\top \b{V}}_{\b{I}} = \b{U} \b{\Sigma} \nonumber\\
&\implies \b{U} = \b{\Psi}\b{V} \b{\Sigma}^{-1} \label{equation_U_dual_SPCA}
\end{align}

The projection of data $\b{X}$ in dual SPCA is:
\begin{align}
\widetilde{\b{X}} &\overset{(\ref{equation_SPCA_projection})}{=} \b{U}^\top \b{X} \overset{(\ref{equation_U_dual_SPCA})}{=} (\b{\Psi}\b{V} \b{\Sigma}^{-1})^\top \b{X} = \b{\Sigma}^{-\top}\b{V}^\top\b{\Psi}^\top\b{X} \nonumber \\
& \overset{(\ref{equation_Psi_dual_SPCA})}{=} \b{\Sigma}^{-1}\b{V}^\top\b{\Delta}^\top\b{H}\b{X}^\top\b{X}. \label{equation_projected_dual_SPCA}
\end{align}
Note that $\b{\Sigma}$ and $\b{H}$ are symmetric. 

Similarly, out-of-sample projection in dual SPCA is:
\begin{align}\label{equation_outOfSample_projected_dual_SPCA}
\widetilde{\b{x}}_t = \b{\Sigma}^{-1}\b{V}^\top\b{\Delta}^\top\b{H}\b{X}^\top\b{x}_t.
\end{align}

Considering all the $n_t$ out-of-sample data points, the projection is:
\begin{align}
\widetilde{\b{X}}_t = \b{\Sigma}^{-1}\b{V}^\top\b{\Delta}^\top\b{H}\b{X}^\top\b{X}_t.
\end{align}

Reconstruction of $\b{X}$ after projection onto the SPCA subspace is:
\begin{align}
\widehat{\b{X}} &\overset{(\ref{equation_SPCA_reconstruction})}{=} \b{U}\b{U}^\top \b{X} = \b{U}\widetilde{\b{X}} \nonumber \\
&\overset{(a)}{=} \b{\Psi}\b{V} \b{\Sigma}^{-1}\b{\Sigma}^{-1}\b{V}^\top\b{\Delta}^\top\b{H}\b{X}^\top\b{X} \nonumber \\
&= \b{\Psi}\b{V} \b{\Sigma}^{-2}\b{V}^\top\b{\Delta}^\top\b{H}\b{X}^\top\b{X} \nonumber \\
&\overset{(\ref{equation_Psi_dual_SPCA})}{=} \b{X}\b{H}\b{\Delta}\b{V} \b{\Sigma}^{-2}\b{V}^\top\b{\Delta}^\top\b{H}\b{X}^\top\b{X} \label{equation_reconstruction_dual_SPCA}
\end{align}
where $(a)$ is because of Eqs. (\ref{equation_U_dual_SPCA}) and (\ref{equation_projected_dual_SPCA}).

Similarly, reconstruction of an out-of-sample data point in dual SPCA is:
\begin{align}\label{equation_outOfSample_reconstruction_dual_SPCA}
\widehat{\b{x}}_t = \b{X}\b{H}\b{\Delta}\b{V} \b{\Sigma}^{-2}\b{V}^\top\b{\Delta}^\top\b{H}\b{X}^\top\b{x}_t.
\end{align}

Considering all the $n_t$ out-of-sample data points, the reconstruction is:
\begin{align}
\widehat{\b{X}}_t = \b{X}\b{H}\b{\Delta}\b{V} \b{\Sigma}^{-2}\b{V}^\top\b{\Delta}^\top\b{H}\b{X}^\top\b{X}_t.
\end{align}

Note that dual PCA was important especially because it provided opportunity to kernelize the PCA. However, as it is explained in the next section, kernel SPCA can be obtained directly from SPCA. Therefore, dual SPCA might not be very important for the sake of kernel SPCA.

The dual SPCA has another benefit similar to what we had for dual PCA. In Eqs. (\ref{equation_projected_dual_SPCA}), (\ref{equation_outOfSample_projected_dual_SPCA}), (\ref{equation_reconstruction_dual_SPCA}), and (\ref{equation_outOfSample_reconstruction_dual_SPCA}), $\b{U}$ is not used but $\b{V}$ exists. In Eq. (\ref{equation_Psi_SVD_dual_SPCA}), the columns of $\b{V}$ are the eigenvectors of $\b{\Psi}^\top \b{\Psi} \in \mathbb{R}^{n \times n}$, according to Proposition \ref{proposition_SVD} in appendix \ref{section_appendix_SVD}. On the other hand, in direct SPCA, we have eigen-decomposition of $\b{X}\b{H}\b{K}_y\b{H}\b{X}^\top \in \mathbb{R}^{d \times d}$ in Eq. (\ref{eqaution_SPCA_eigendecomposition}) which is then used in Eqs. (\ref{equation_SPCA_projection}), (\ref{equation_SPCA_projection_outOfSample}), (\ref{equation_SPCA_reconstruction}), and (\ref{equation_SPCA_reconstruction_outOfSample}).
In case we have huge dimensionality, $d \gg n$, decomposition of the $n \times n$ matrix is faster and needs less storage so dual SPCA will be more efficient.

\subsection{Kernel Supervised PCA}

The SPCA can be kernelized by two approaches, using either direct SPCA or dual SPCA \cite{barshan2011supervised}. 

\subsubsection{Kernel SPCA Using Direct SPCA}

According to the representation theory \cite{alperin1993local}, any solution (direction) $\b{u} \in \mathcal{H}$ must lie in the span of ``all'' the training vectors mapped to $\mathcal{H}$, i.e., $\b{\Phi}(\b{X}) = [\b{\phi}(\b{x}_1), \dots, \b{\phi}(\b{x}_n)] \in \mathbb{R}^{t\times n}$ (usually $t \gg d$). Note that $\mathcal{H}$ denotes the Hilbert space (feature space). Therefore, we can state that:
\begin{align*}
\b{u} = \sum_{i=1}^n \theta_i\, \b{\phi}(\b{x}_i) = \b{\Phi}(\b{X})\, \b{\theta},
\end{align*}
where $\b{\theta} \in \mathbb{R}^n$ is the unknown vector of coefficients, and $\b{u} \in \mathbb{R}^t$ is the kernel SPCA direction in Hilbert space here. 
The directions can be put together in $\mathbb{R}^{t \times p} \ni \b{U} := [\b{u}_1, \dots, \b{u}_p]$:
\begin{align}\label{equation_U_kernel_SPCA}
\b{U} = \b{\Phi}(\b{X})\, \b{\Theta},
\end{align}
where $\b{\Theta} := [\b{\theta}_1, \dots, \b{\theta}_p] \in \mathbb{R}^{n \times p}$.

The Eq. (\ref{equation_HSIC_SPCA}) in the feature space becomes:
\begin{align*}
\text{HSIC} = \frac{1}{(n-1)^2}\, \textbf{tr}(\b{\Phi}(\b{X})^\top \b{U} \b{U}^\top \b{\Phi}(\b{X})\b{H}\b{K}_y\b{H}).
\end{align*}
The $\textbf{tr}(\b{\Phi}(\b{X})^\top \b{U} \b{U}^\top \b{\Phi}(\b{X})\b{H}\b{K}_y\b{H})$ can be simplified as:
\begin{align}
&\textbf{tr}(\b{\Phi}(\b{X})^\top \b{U} \b{U}^\top \b{\Phi}(\b{X})\b{H}\b{K}_y\b{H}) \nonumber \\
&= \textbf{tr}(\b{U} \b{U}^\top \b{\Phi}(\b{X})\b{H}\b{K}_y\b{H}\b{\Phi}(\b{X})^\top) \nonumber \\
&= \textbf{tr}(\b{U}^\top \b{\Phi}(\b{X})\b{H}\b{K}_y\b{H}\b{\Phi}(\b{X})^\top\b{U}) \label{equation_traceInHSIC_kernel_SPCA}
\end{align}
Plugging Eq. (\ref{equation_U_kernel_SPCA}) in Eq. (\ref{equation_traceInHSIC_kernel_SPCA}) gives us:
\begin{align}
&\textbf{tr}\big(\b{\Theta}^\top \b{\Phi}(\b{X})^\top \b{\Phi}(\b{X})\b{H}\b{K}_y\b{H}\b{\Phi}(\b{X})^\top\b{\Phi}(\b{X})\, \b{\Theta}\big) \nonumber \\
&= \textbf{tr}(\b{\Theta}^\top \b{K}_x \b{H}\b{K}_y\b{H}\b{K}_x \b{\Theta}),
\end{align}
where:
\begin{align}\label{equation_kernel_SPCA_kernel_X}
\mathbb{R}^{n \times n} \ni \b{K}_x := \b{\Phi}(\b{X})^\top \b{\Phi}(\b{X}).
\end{align}
Note that the Eqs. (\ref{equation_kernel_SPCA_kernel_X}) and (\ref{equation_SPCA_kernel_of_projection}) are different and should not be confused.

Moreover, the constraint of orthogonality of projection matrix, i.e., $\b{U}^\top \b{U} = \b{I}$, in the feature space becomes:
\begin{align}
\b{U}^\top \b{U} &= \big(\b{\Phi}(\b{X})\, \b{\Theta}\big)^\top \big(\b{\Phi}(\b{X})\, \b{\Theta}\big) \nonumber \\
&= \b{\Theta}^\top \b{\Phi}(\b{X})^\top \b{\Phi}(\b{X}) \b{\Theta} = \b{\Theta}^\top \b{K}_x\, \b{\Theta}.
\end{align}
Therefore, the optimization problem is:
\begin{equation}
\begin{aligned}
& \underset{\b{\Theta}}{\text{maximize}}
& & \textbf{tr}(\b{\Theta}^\top \b{K}_x \b{H}\b{K}_y\b{H}\b{K}_x \b{\Theta}), \\
& \text{subject to}
& & \b{\Theta}^\top \b{K}_x\, \b{\Theta} = \b{I},
\end{aligned}
\end{equation}
where the objective variable is the unknown $\b{\Theta}$.

Using Lagrange multiplier \cite{boyd2004convex}, we have:
\begin{align*}
&\mathcal{L} = \\
&\textbf{tr}(\b{\Theta}^\top \b{K}_x \b{H}\b{K}_y\b{H}\b{K}_x \b{\Theta}) - \textbf{tr}\big(\b{\Lambda}^\top (\b{\Theta}^\top \b{K}_x\, \b{\Theta} - \b{I})\big) \\
&= \textbf{tr}(\b{\Theta} \b{\Theta}^\top \b{K}_x \b{H}\b{K}_y\b{H}\b{K}_x) - \textbf{tr}\big(\b{\Lambda}^\top (\b{\Theta}^\top \b{K}_x\, \b{\Theta} - \b{I})\big),
\end{align*}
where $\b{\Lambda} \in \mathbb{R}^{p \times p}$ is a diagonal matrix $\textbf{diag}([\lambda_1, \dots, \lambda_p]^\top)$.
\begin{align}
& \mathbb{R}^{n \times p} \ni \frac{\partial \mathcal{L}}{\partial \b{\Theta}} = 2 \b{K}_x \b{H}\b{K}_y\b{H}\b{K}_x \b{\Theta} - 2\b{K}_x\b{\Theta} \b{\Lambda} \overset{\text{set}}{=} 0 \nonumber \\ 
& \implies \b{K}_x \b{H}\b{K}_y\b{H}\b{K}_x \b{\Theta} = \b{K}_x\b{\Theta} \b{\Lambda}, \label{equation_kernel_SPCA_generalized_eigen_problem}
\end{align}
which is the generalized eigenvalue problem $(\b{K}_x \b{H}\b{K}_y\b{H}\b{K}_x, \b{K}_x)$ \cite{ghojogh2019eigenvalue}. The $\b{\Theta}$ and $\b{\Lambda}$, which are the eigenvector and eigenvalue matrices, respectively, can be calculated according to \cite{ghojogh2019eigenvalue}.

Note that in practice, we can naively solve Eq. (\ref{equation_kernel_SPCA_generalized_eigen_problem}) by left multiplying $\b{K}_x^{-1}$ (hoping that it is positive definite and thus not singular):
\begin{align}
&\underbrace{\b{K}_x^{-1}\b{K}_x}_{\b{I}} \b{H}\b{K}_y\b{H}\b{K}_x \b{\Theta} = \b{\Theta} \b{\Lambda} \nonumber \\
&\implies \b{H}\b{K}_y\b{H}\b{K}_x \b{\Theta} = \b{\Theta} \b{\Lambda},
\end{align}
which is the eigenvalue problem \cite{ghojogh2019eigenvalue} for $\b{H}\b{K}_y\b{H}\b{K}_x$, where columns of $\b{\Theta}$ are the eigenvectors of it and $\b{\Lambda}$ includes its eigenvalues on its diagonal. 

If we take the $p$ leading eigenvectors to have $\b{\Theta} \in \mathbb{R}^{n \times p}$, the projection of $\b{\Phi}(\b{X}) \in \mathbb{R}^{t \times n}$ is:
\begin{align}
\mathbb{R}^{p \times n} \ni \b{\Phi}(\widetilde{\b{X}}) &= \b{U}^\top \b{\Phi}(\b{X}) \nonumber \\
&\overset{(\ref{equation_U_kernel_SPCA})}{=} \b{\Theta}^\top \b{\Phi}(\b{X})^\top \b{\Phi}(\b{X}) = \b{\Theta}^\top \b{K}_x, 
\end{align}
where $\mathbb{R}^{n \times n} \ni \b{K}_x := \b{\Phi}(\b{X})^\top \b{\Phi}(\b{X})$.
Similarly, the projection of out-of-sample data point $\b{\phi}(\b{x}_t) \in \mathbb{R}^t$ is:
\begin{align}
\mathbb{R}^{p} \ni \b{\phi}(\widetilde{\b{x}}_t) &= \b{U}^\top \b{\phi}(\b{x}_t) \nonumber \\
&\overset{(\ref{equation_U_kernel_SPCA})}{=} \b{\Theta}^\top \b{\Phi}(\b{X})^\top \b{\phi}(\b{x}_t) = \b{\Theta}^\top \b{k}_t, 
\end{align}
where $\b{k}_t$ is Eq. (\ref{equation_appendix_kernelVector_outOfSample}).

Considering all the $n_t$ out-of-sample data points, $\b{X}_t$, the projection is:
\begin{align}
\mathbb{R}^{p \times n_t} \ni \b{\phi}(\widetilde{\b{X}}_t) = \b{\Theta}^\top \b{K}_t,
\end{align}
where $\b{K}_t$ is Eq. (\ref{equation_appendix_kernelMatrix_outOfSample}).

As we will show in the following section, in kernel SPCA, as in kernel PCA, we cannot reconstruct data, whether training or out-of-sample. 

\subsubsection{Kernel SPCA Using Dual SPCA}

The Eq. (\ref{equation_Psi_dual_SPCA}) in $t$-dimensional feature space becomes:
\begin{align}\label{equation_Psi_kernel_SPCA_using_dual}
\mathbb{R}^{t \times n} \ni \b{\Psi} = \b{\Phi}(\b{X})\b{H}\b{\Delta},
\end{align}
where $\b{\Phi}(\b{X}) = [\b{\phi}(\b{x}_1), \dots, \b{\phi}(\b{x}_n)] \in \mathbb{R}^{t \times n}$.

Applying SVD (see Appendix \ref{section_appendix_SVD}) on $\b{\Psi}$ of Eq. (\ref{equation_Psi_kernel_SPCA_using_dual}) is similar to the form of Eq. (\ref{equation_Psi_SVD_dual_SPCA}).
Having the same discussion which we had for Eqs. (\ref{equation_kernelPCA_SVD_Phi}) and (\ref{equation_kernelPCA_eigen_centered_kernel_2}), we do not necessarily have $\b{\Phi}(\b{X})$ in Eq. (\ref{equation_Psi_kernel_SPCA_using_dual}) so we can obtain $\b{V}$ and $\b{\Sigma}$ as:
\begin{align}\label{equation_kernel_SPCA_eigen_Delta_centered_kernel}
\big(\b{\Delta}^\top \breve{\b{K}}_x \b{\Delta}\big) \b{V} = \b{V} \b{\Sigma}^2,
\end{align}
where $\breve{\b{K}}_x := \b{H} \b{K}_x \b{H} \b{\Delta}$ and the columns of $\b{V}$ are the eigenvectors of (see Proposition \ref{proposition_SVD} in Appendix \ref{section_appendix_SVD}):
\begin{align*}
\b{\Psi}^\top \b{\Psi} &\overset{(a)}{=} \b{\Delta}^\top \b{H} \b{\Phi}(\b{X})^\top \b{\Phi}(\b{X}) \b{H} \b{\Delta} \overset{(\ref{equation_kernel_SPCA_kernel_X})}{=} \b{\Delta}^\top \b{H} \b{K}_x \b{H} \b{\Delta} \\
&= \b{\Delta}^\top \breve{\b{K}}_x \b{\Delta},
\end{align*}
where $(a)$ is because of Eqs. (\ref{equation_Psi_kernel_SPCA_using_dual}) and (\ref{equation_centeringMatrix_is_symmetric}).

It is noteworthy that because of using Eq. (\ref{equation_kernel_SPCA_eigen_Delta_centered_kernel}) instead of Eq. (\ref{equation_Psi_kernel_SPCA_using_dual}), \textit{the projection directions $\b{U}$ are not available in kernel SPCA to be observed or plotted.}

Similar to equations (\ref{equation_Psi_SVD_dual_SPCA}) and (\ref{equation_U_dual_SPCA}), we have:
\begin{align}
\b{\Psi} = \b{U} \b{\Sigma} \b{V}^\top &\implies \b{\Psi}\b{V} = \b{U} \b{\Sigma} \underbrace{\b{V}^\top \b{V}}_{\b{I}} = \b{U} \b{\Sigma} \nonumber\\
&\implies \b{U} = \b{\Psi}\b{V} \b{\Sigma}^{-1}, \label{equation_U_kernel_SPCA_using_dual}
\end{align}
where $\b{V}$ and $\b{\Sigma}$ are obtained from Eq. (\ref{equation_kernel_SPCA_eigen_Delta_centered_kernel}).

The projection of data $\b{\Phi}(\b{X})$ is:
\begin{align}
\b{\Phi}(\widetilde{\b{X}}) &= \b{U}^\top \b{\Phi}(\b{X}) = (\b{\Psi}\b{V} \b{\Sigma}^{-1})^\top \b{\Phi}(\b{X}) \nonumber \\
&= \b{\Sigma}^{-\top}\b{V}^\top\b{\Psi}^\top\b{\Phi}(\b{X}) \nonumber \nonumber \\
& \overset{(\ref{equation_Psi_kernel_SPCA_using_dual})}{=} \b{\Sigma}^{-1}\b{V}^\top\b{\Delta}^\top\b{H}\b{\Phi}(\b{X})^\top\b{\Phi}(\b{X}), \nonumber \\
& \overset{(\ref{equation_kernel_SPCA_kernel_X})}{=} \b{\Sigma}^{-1}\b{V}^\top\b{\Delta}^\top\b{H}\b{K}_x.
\label{equation_projected_kernel_SPCA_using_dual}
\end{align}
Note that $\b{\Sigma}$ and $\b{H}$ are symmetric. 

Similarly, out-of-sample projection in kernel SPCA is:
\begin{align}
\b{\phi}(\widetilde{\b{x}}_t) &= \b{\Sigma}^{-1}\b{V}^\top\b{\Delta}^\top\b{H}\b{\Phi}(\b{X})^\top\b{\phi}(\b{x}_t) \nonumber \\
&= \b{\Sigma}^{-1}\b{V}^\top\b{\Delta}^\top\b{H}\,\b{k}_t,
\end{align}
where $\b{k}_t$ is Eq. (\ref{equation_appendix_kernelVector_outOfSample}).

Considering all the $n_t$ out-of-sample data points, $\b{X}_t$, the projection is:
\begin{align}
\b{\phi}(\widetilde{\b{X}}_t)
= \b{\Sigma}^{-1}\b{V}^\top\b{\Delta}^\top\b{H}\,\b{K}_t.
\end{align}
where $\b{K}_t$ is Eq. (\ref{equation_appendix_kernelMatrix_outOfSample}).

Reconstruction of $\b{\Phi}(\b{X})$ after projection onto the SPCA subspace is:
\begin{align}
\b{\Phi}(\widehat{\b{X}}) &= \b{U}\b{U}^\top \b{\Phi}(\b{X}) = \b{U}\b{\Phi}(\widetilde{\b{X}}) \nonumber \\
& \overset{(a)}{=} \b{\Psi}\b{V} \b{\Sigma}^{-1}\b{\Sigma}^{-1}\b{V}^\top\b{\Delta}^\top\b{H}\b{K}_x \nonumber \\
&= \b{\Psi}\b{V} \b{\Sigma}^{-2}\b{V}^\top\b{\Delta}^\top\b{H}\b{K}_x \nonumber \\
&\overset{(\ref{equation_Psi_kernel_SPCA_using_dual})}{=} \b{\Phi}(\b{X})\b{H}\b{\Delta}\b{V} \b{\Sigma}^{-2}\b{V}^\top\b{\Delta}^\top\b{H}\b{K}_x \label{equation_outOfSample_projected_kernel_SPCA_using_dual}
\end{align}
where $(a)$ is because of Eqs. (\ref{equation_U_kernel_SPCA_using_dual}) and (\ref{equation_projected_kernel_SPCA_using_dual}).

Similarly, reconstruction of an out-of-sample data point in dual SPCA is:
\begin{align}
\widehat{\b{x}}_t &= \b{\Phi}(\b{X})\b{H}\b{\Delta}\b{V} \b{\Sigma}^{-2}\b{V}^\top\b{\Delta}^\top\b{H}\b{\Phi}(\b{X})^\top\b{\phi}(\b{x}_t) \nonumber\\
&= \b{\Phi}(\b{X})\b{H}\b{\Delta}\b{V} \b{\Sigma}^{-2}\b{V}^\top\b{\Delta}^\top\b{H}\,\b{k}_t, \label{equation_outOfSample_reconstruction_kernel_SPCA_using_dual}
\end{align}
where $\b{k}_t$ is Eq. (\ref{equation_appendix_kernelVector_outOfSample}).

However, in Eqs. (\ref{equation_outOfSample_projected_kernel_SPCA_using_dual}) and (\ref{equation_outOfSample_reconstruction_kernel_SPCA_using_dual}), we do not necessarily have $\b{\Phi}(\b{X})$; therefore, in kernel SPCA, as in kernel PCA, we cannot reconstruct data, whether training or out-of-sample.

\begin{figure*}[!t]
\centering
\includegraphics[width=6.5in]{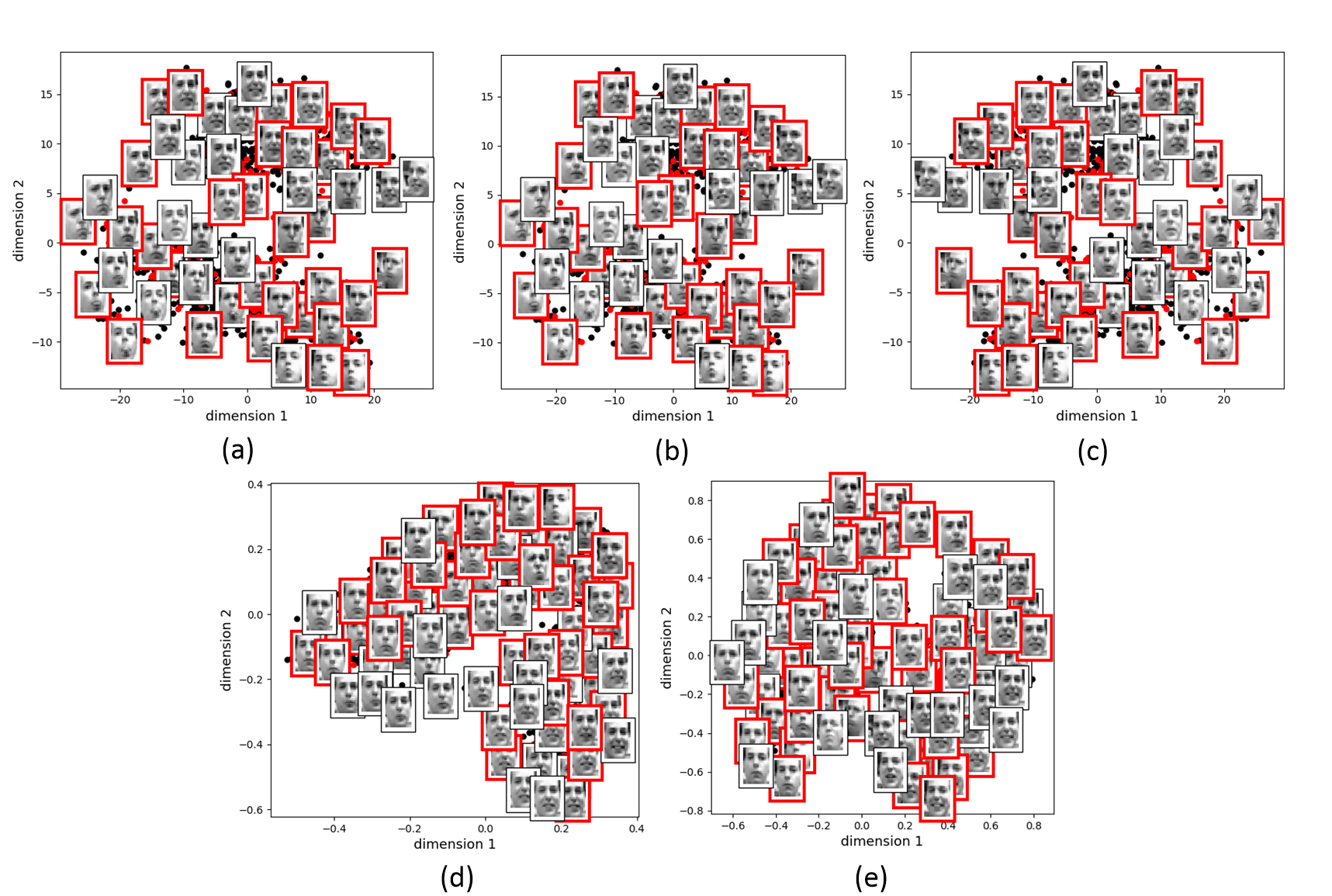}
\caption{Projection of the training and out-of-sample sets of Frey dataset onto the subspaces of (a) the PCA, (b) the dual SPCA, (c) the kernel PCA (linear kernel), (d) the kernel PCA (RBF kernel), and (e) the kernel PCA (cosine kernel). The images with red frame are the out-of-sample images.}
\label{figure_projection_test_Frey}
\end{figure*}

\begin{figure*}[!t]
\centering
\includegraphics[width=6in]{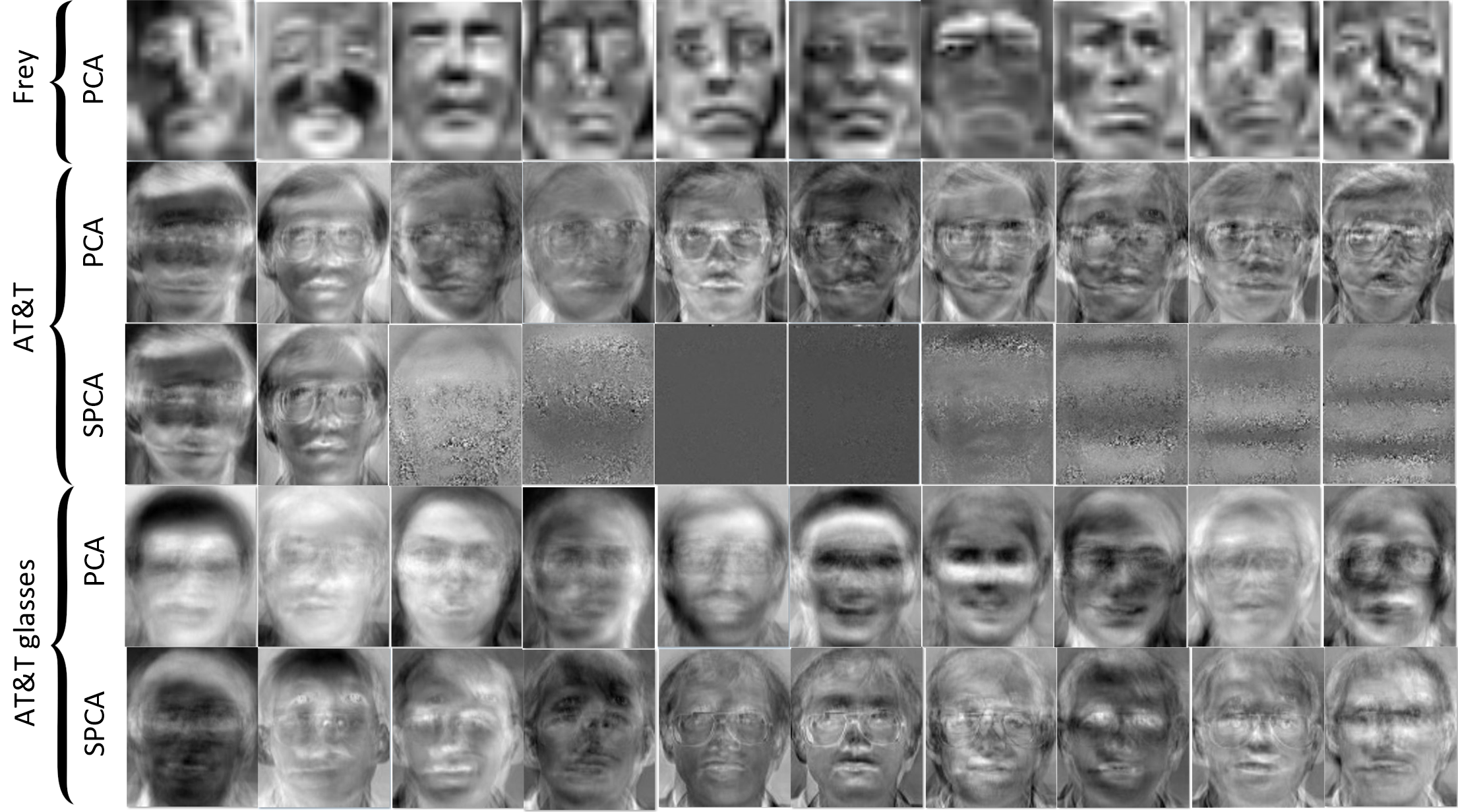}
\caption{The ghost faces: the leading eigenvectors of PCA and SPCA for Frey, AT\&T, and AT\&T glasses datasets.}
\label{figure_projection_directions}
\end{figure*}

\begin{figure*}[!t]
\centering
\includegraphics[width=4in]{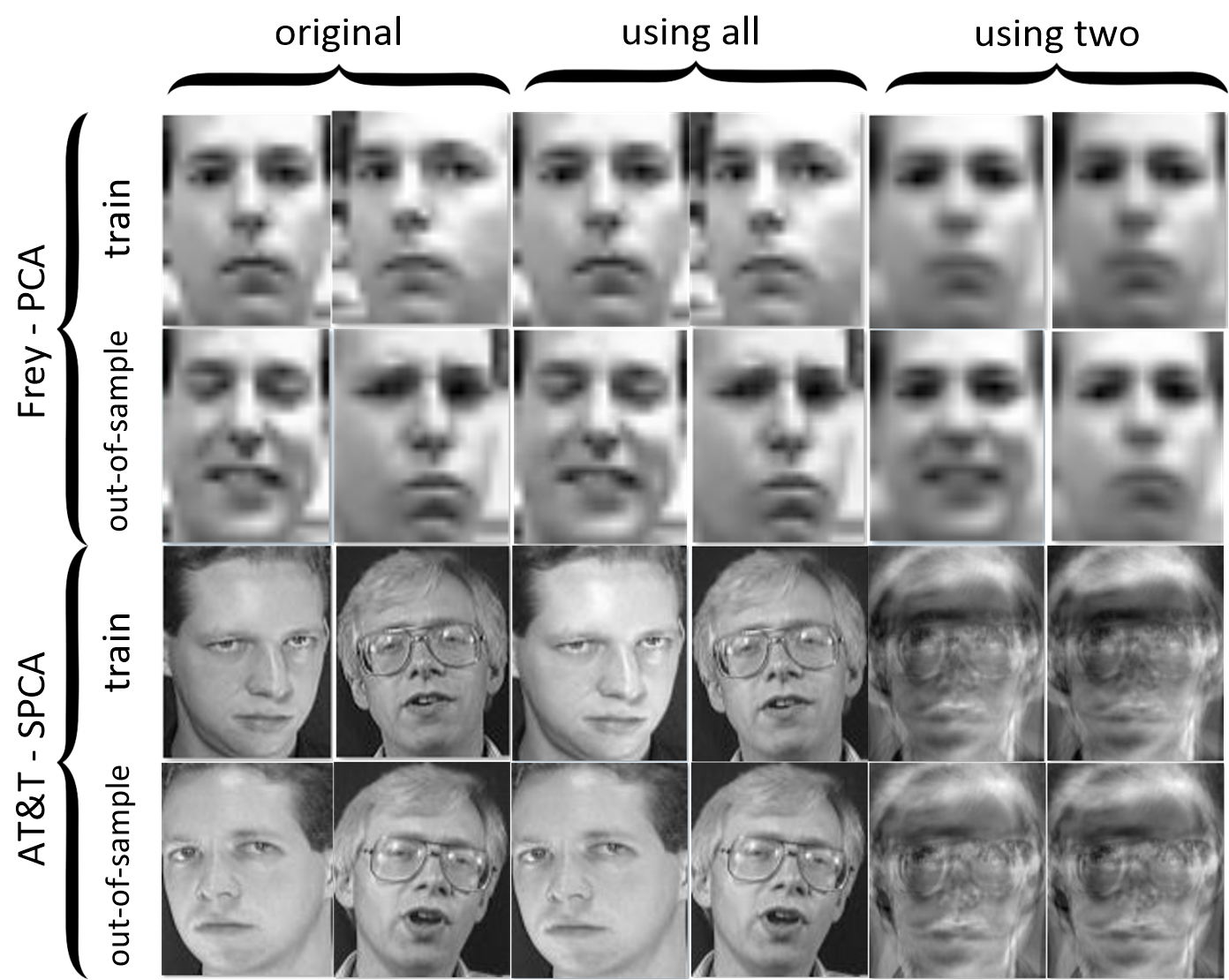}
\caption{The reconstructed faces using all and two of the leading eigenvectors of PCA and SPCA for Frey and AT\&T datasets.}
\label{figure_reconstruction_images}
\end{figure*}

\section{Eigenfaces}

This section introduces one of the most fundamental applications of PCA and its variants -- facial recognition.

\subsection{Projection Directions of Facial Images}

PCA and kernel PCA can be trained using images of diverse faces, to learn the most important facial features, which account for the variation between faces. Here, two facial datasets, i.e. the Frey dataset and the AT\&T (ORL) face dataset, are used to illustrate this concept. The AT\&T dataset has been used twice, i.e., (1) with human subjects as its classes and (2) with having and not having eye glasses as its classes.
Figure \ref{figure_projection_directions} demonstrates the top ten PCA directions for the PCA trained on these datasets.
As demonstrated, the projection directions of a facial dataset are some facial features which are like ghost faces in terms of appearance. That is why the facial projection directions are also referred to as ``ghost faces''.
The ghost faces in PCA are also referred to as ``eigenfaces'' \cite{turk1991eigenfaces,turk1991face} because PCA uses eigenvalue decomposition of the covariance matrix.

In Fig. \ref{figure_projection_directions}, the projection directions have captured different facial features that discriminate the data with respect to the maximum variance. The captured features are eyes, nose, cheeks, chin, lips, and eyebrows, which are the most important facial features.
This figure does not include projection directions of the kernel PCA because in kernel PCA the projection directions are not available.
Note that the facial recognition using the kernel PCA is referred to as ``kernel eigenfaces'' \cite{yang2000face}.
The ghost faces (facial projection directions) of SPCA can be referred to as the ``supervised eigenfaces''. Facial recognition using the kernel SPCA can also be referred to as ``kernel supervised eigenfaces''.
Figure \ref{figure_projection_directions} does not include projection directions of the kernel SPCA because the projection directions are not available in kernel SPCA.

Comparison of PCA and SPCA directions demonstrates that both PCA and SPCA are capturing eye glasses as important discriminators. However, some Haar wavelet\footnote{Haar wavelet is a family of square-shaped filters which form wavelet bases.} like features \cite{stankovic2003haar} are captured as the projection directions in SPCA. Haar wavelets are important in face recognition and detection; for example, they have been used in the Viola-Jones face detector \cite{wang2014analysis}.
As demonstrated in Fig. \ref{figure_projection_directions}, both PCA and SPCA have captured eyes as discriminators; however, SPCA has also focused on the frame of eye glasses because of the usage of class labels. Where PCA has also captured other distracting facial features, such as forehead, cheeks, hair, mustache, etc, because it is not aware that the two classes are different, in terms of glasses, and sees the dataset as a whole.

\subsection{Projection of Facial Images}

Using the obtained projection directions, the facial images can be projected onto the PCA subspace. Similarly, projected images using kernel PCA can also be obtained.
Figure \ref{figure_projection_test_Frey} demonstrates the projection of both training and out-of-sample facial images, of the Frey dataset onto the PCA, dual PCA, and kernel PCA subspaces. The used kernels were linear, RBF, and cosine.
As can be seen, the out-of-sample data, although were not seen in the training phase, are projected very well. The model, somewhat, has extrapolated the projections so it has learned generalizable subspaces.

\subsection{Reconstruction of Facial Images}

The facial images can be reconstructed after projection onto PCA and SPCA subspaces.
The reconstruction of training and test images, in Frey and AT\&T datasets, are depicted in Fig. \ref{figure_reconstruction_images}.
Reconstructions have occurred using all and also two top projection directions. 
As expected, the reconstructions using all projection directions are very similar to the original images.
However, reconstruction using two leading projection directions is not prefect. Although, most important facial features are reconstructed because the leading projection directions carry most of the information.

\section{Conclusion and Future Work}\label{section_conclusions}

In this paper, the PCA and SPCA were introduced in details of theory. Moreover, kernel PCA and kernel SPCA were covered. The illustrations and experiments on Frey and AT\&T face datasets were also provided in order to analyze the explained methods in practice. 

The calculation of $\b{K}_y \in \mathbb{R}^{n \times n}$ in SPCA might be challenging for big data in terms of speed and storage. The Supervised Random Projection (SRP) \cite{karimi2018srp,karimi2018exploring} addresses this problem by approximating the kernel matrix $\b{K}_y$ using Random Fourier Features (RFF)
\cite{rahimi2008random}. As a future work, we will write a tutorial on SRP.

Moreover, the sparsity is very effective because of the \textit{``bet on sparsity''} principal: ``Use a procedure that does well in sparse problems, since no procedure does well in dense problems \cite{friedman2001elements,tibshirani2015statistical}.''
Another reason for the effectiveness of the sparsity is Occam's razor \cite{domingos1999role} stating that ``simpler solutions are more likely to be correct than complex ones'' or ``simplicity is a goal in itself''.
Therefore, the sparse methods such as sparse PCA \cite{zou2006sparse,shen2008sparse}, sparse kernel PCA \cite{tipping2001sparse}, and Sparse Supervised Principal Component Analysis (SSPCA) \cite{sharifzadeh2017sparse} have been proposed. We will defer these methods to future tutorials.

\section*{Acknowledgment}
The authors hugely thank Prof. Ali Ghodsi (see his great online courses \cite{web_data_visualization,web_classification}), Prof. Mu Zhu, Prof. Wayne Oldford, Prof. Hoda Mohammadzade, and other professors whose courses have partly covered the materials mentioned in this tutorial paper.  


\appendix

\section{Centering Matrix}\label{section_appendix_centering}

Consider a matrix $\b{A} \in \mathbb{R}^{\alpha \times \beta}$.
We can show this matrix by its rows, $\b{A} = [\b{a}_1, \dots, \b{a}_\alpha]^\top$ or by its columns, $\b{A} = [\b{b}_1, \dots, \b{b}_\beta]$, where $\b{a}_i$ and $\b{b}_j$ denotes the $i$-th row and $j$-th column of $\b{A}$, respectively.
Note that the vectors are column vectors.

The ``left centering matrix'' is defined as:
\begin{align}\label{equation_left_centering_matrix}
\mathbb{R}^{\alpha \times \alpha} \ni \b{H} := \b{I} - (1/\alpha) \b{1}\b{1}^\top,
\end{align}
where $\b{1} = [1, \dots, 1]^\top \in \mathbb{R}^{\alpha}$ and $\b{I} \in \mathbb{R}^{\alpha \times \alpha}$ is the identity matrix.
Left multiplying this matrix to $\b{A}$, i.e., $\b{H}\b{A}$, removes the mean of rows of $\b{A}$ from all of its rows:
\begin{align}\label{equation_left_centered_matrix}
\b{H}\b{A} \overset{(\ref{equation_left_centering_matrix})}{=} \b{A} - (1/\alpha) \b{1}\b{1}^\top \b{A} = (\b{A}^\top - \b{\mu}_{\text{rows}})^\top,
\end{align}
where the column vector $\b{\mu}_{\text{rows}} \in \mathbb{R}^\beta$ is the mean of rows of $\b{A}$.

The ``right centering matrix'' is defined as:
\begin{align}\label{equation_right_centering_matrix}
\mathbb{R}^{\beta \times \beta} \ni \b{H} := \b{I} - (1/\beta) \b{1}\b{1}^\top,
\end{align}
where $\b{1} = [1, \dots, 1]^\top \in \mathbb{R}^{\beta}$ and $\b{I} \in \mathbb{R}^{\beta \times \beta}$ is the identity matrix.
Right multiplying this matrix to $\b{A}$, i.e., $\b{A}\b{H}$, removes the mean of columns of $\b{A}$ from all of its columns:
\begin{align}\label{equation_right_centered_matrix}
\b{A}\b{H} \overset{(\ref{equation_right_centering_matrix})}{=} \b{A} - (1/\beta) \b{A} \b{1}\b{1}^\top = \b{A} - \b{\mu}_{\text{cols}}
\end{align}
where the column vector $\b{\mu}_{\text{cols}} \in \mathbb{R}^\alpha$ is the mean of columns of $\b{A}$.

We can use both left and right centering matrices at the same time:
\begin{align}
\b{H}\b{A}\b{H} &= (\b{I}_\alpha - (1/\alpha) \b{1}_\alpha\b{1}_\alpha^\top) \b{A} (\b{I}_\beta - (1/\beta) \b{1}_\beta\b{1}_\beta^\top) \nonumber \\
&= (\b{A} - (1/\alpha) \b{1}_\alpha\b{1}_\alpha^\top \b{A}) (\b{I}_\beta - (1/\beta) \b{1}_\beta\b{1}_\beta^\top) \nonumber \\
&= \b{A} - (1/\alpha) \b{1}_\alpha\b{1}_\alpha^\top \b{A} - (1/\beta) \b{A} \b{1}_\beta\b{1}_\beta^\top \nonumber \\
&~~~~~ + (1 / (\alpha\beta)) \b{1}_\alpha\b{1}_\alpha^\top \b{A} \b{1}_\beta\b{1}_\beta^\top. \label{equation_appendix_doubleCentered_A}
\end{align}
This operation is commonly done for a kernel (see appendix A in \cite{scholkopf1998nonlinear} and Appendix \ref{section_appendix_centeringKernel} in this tutorial paper).
The second term removes the mean of rows of $\b{A}$ according to Eq. (\ref{equation_left_centered_matrix}) and the third term removes the mean of columns of $\b{A}$ according to Eq. (\ref{equation_right_centered_matrix}). The last term, however, adds the overall mean of $\b{A}$ back to it where the matrix $\b{\mu}_{\text{all}} \in \mathbb{R}^{\alpha \times \beta}$ whose all elements are the overall mean of $\b{A}$ is:
\begin{align}
&\b{\mu}_{\text{all}} := (1 / (\alpha\beta)) \b{1}_\alpha\b{1}_\alpha^\top \b{A} \b{1}_\beta\b{1}_\beta^\top \\
&\b{\mu}_{\text{all}}(.,.) = \frac{1}{\alpha \beta} \sum_{i=1}^\alpha \sum_{j=1}^\beta \b{A}(i,j),
\end{align}
where $\b{A}(i,j)$ is the $(i,j)$-th element of $\b{A}$ and $\b{\mu}_{\text{all}}(.,.)$ is every element of $\b{A}$.

Therefore, ``double centering'' for $\b{A}$ is defined as:
\begin{align}
\b{H}\b{A}\b{H} = (\b{A}^\top - \b{\mu}_{\text{rows}})^\top - \b{\mu}_{\text{cols}} + \b{\mu}_{\text{all}},
\end{align}
which removes both the row and column means of $\b{A}$ but adds back the overall mean.
Note that if the matrix $\b{A}$ is a square matrix, the left and right centering matrices are equal with the same dimensionality as the matrix $\b{A}$.

In computer programming, usage of centering matrix might have some precision errors. therefore, in computer programming, we have:
\begin{align*}
&\b{H}\b{A} \approx (\b{A}^\top - \b{\mu}_{\text{rows}})^\top, \\
&\b{A}\b{H} \approx \b{A} - \b{\mu}_{\text{cols}}, \\
&\b{H}\b{A}\b{H} \approx (\b{A}^\top - \b{\mu}_{\text{rows}})^\top - \b{\mu}_{\text{cols}} + \b{\mu}_{\text{all}},
\end{align*}
with a good enough approximation.

Moreover, the centering matrix is symmetric because:
\begin{align}
\b{H}^\top &= (\b{I} - (1/\alpha) \b{1}\b{1}^\top)^\top = \b{I} ^\top- (1/\alpha) (\b{1}\b{1}^\top)^\top \nonumber \\
&= \b{I} - (1/\alpha) \b{1}\b{1}^\top \overset{(\ref{equation_left_centering_matrix})}{=} \b{H}. \label{equation_centeringMatrix_is_symmetric}
\end{align}

The centering matrix is also idempotent:
\begin{align}\label{equation_centeringMatrix_is_idempotent}
\b{H}^k = \underbrace{\b{H} \b{H} \cdots \b{H}}_{k \text{ times}} = \b{H},
\end{align}
where $k$ is a positive integer. The proof is:
\begin{align*}
\b{H}\b{H} &= (\b{I} - (1/\alpha) \b{1}\b{1}^\top) (\b{I} - (1/\alpha) \b{1}\b{1}^\top) \\
&= \b{I} - (1/\alpha) \b{1}\b{1}^\top - (1/\alpha) \b{1}\b{1}^\top + (1/\alpha^2) \b{1}\underbrace{\b{1}^\top\b{1}}_{\alpha}\b{1}^\top \\
&= \b{I} - (1/\alpha) \b{1}\b{1}^\top - (1/\alpha) \b{1}\b{1}^\top + (1/\alpha) \b{1}\b{1}^\top \\
&= \b{I} - (1/\alpha) \b{1}\b{1}^\top \overset{(\ref{equation_left_centering_matrix})}{=} \b{H}.
\end{align*}
Hence:
\begin{align*}
\b{H}^k = (\underbrace{\b{H} \cdots (\underbrace{\b{H} (\underbrace{\b{H} \b{H}}_{\b{H}}}_{\b{H}}}_{\b{H}}))))))) = \b{H}. ~~~~ \text{Q.E.D.}
\end{align*}

For illustration, we provide a simple example:
\begin{align*}
\b{A} =
\begin{bmatrix}
    1 & 2 & 3 \\
    4 & 3 & 1 \\
\end{bmatrix},
\end{align*}
whose row mean, column mean, and overall mean matrix are:
\begin{align*}
&\b{\mu}_{\text{rows}} = [2.5, 2.5, 2]^\top, \\
&\b{\mu}_{\text{cols}} = [2, 2.66]^\top, \\
&\b{\mu}_{\text{all}} = 
\begin{bmatrix}
    2.33 & 2.33 & 2.33 \\
    2.33 & 2.33 & 2.33 \\
\end{bmatrix},
\end{align*}
respectively. 
The left, right, and double centering of $\b{A}$ are:
\begin{align*}
&\b{H}\b{A} =  
\begin{bmatrix}
    -1.5 & -0.5 & 1 \\
    1.5 & 0.5 & -1 \\
\end{bmatrix}, \\
&\b{A}\b{H} =  
\begin{bmatrix}
    -1 & 0 & 1 \\
    1.34 & 0.34 & -1.66 \\
\end{bmatrix}, \\
&\b{H}\b{A}\b{H} =  \b{H}\b{A} - \b{\mu}_{\text{cols}} + \b{\mu}_{\text{all}} \\
&~~~~~~~~~~~~ = 
\begin{bmatrix}
    -1.17 & -0.17 & 1.33 \\
    1.17 & 0.17 & -1.33 \\
\end{bmatrix},
\end{align*}
respectively.

\section{Singular Value Decomposition}\label{section_appendix_SVD}

Consider a matrix $\b{A} \in \mathbb{R}^{\alpha \times \beta}$.
Singular Value Decomposition (SVD) \cite{stewart1993early} is one of the most well-known and effective matrix decomposition methods. It has two different forms, i.e., complete and incomplete.
There are different methods for obtaining this decomposition, one of which is Jordan's algorithm \cite{stewart1993early}. Here, we do not explain how to obtain SVD but we introduce different forms of SVD and their properties.

The ``complete SVD'' decomposes the matrix as:
\begin{align}
&\mathbb{R}^{\alpha \times \beta} \ni \b{A} = \b{U}\b{\Sigma}\b{V}^\top, \\
&\b{U} \in \mathbb{R}^{\alpha \times \alpha}, ~~ \b{V} \in \mathbb{R}^{\beta \times \beta}, ~~ \b{\Sigma} \in \mathbb{R}^{\alpha \times \beta}, \nonumber 
\end{align}
where the columns of $\b{U}$ and the columns of $\b{V}$ are called ``left singular vectors'' and ``right singular vectors'', respectively. 
In complete SVD, the $\b{\Sigma}$ is a \textit{rectangular} diagonal matrix whose main diagonal includes the ``singular values''. In cases $\alpha > \beta$ and $\alpha < \beta$, this matrix is in the forms:
\begin{align*}
\b{\Sigma} = 
\begin{bmatrix}
    \sigma_{1} & 0 & 0 \\
    \vdots & \ddots & \vdots \\
    0 & 0 & \sigma_{\beta} \\
    0 & 0 & 0 \\
    \vdots & \vdots & \vdots \\
    0 & 0 & 0
\end{bmatrix}
\text{and}
\begin{bmatrix}
    \sigma_{1} & 0 & 0 & 0 & \cdots & 0 \\
    \vdots & \ddots & \vdots & 0 & \cdots & 0 \\
    0 & 0 & \sigma_{\alpha} & 0 & \cdots & 0 \\
\end{bmatrix},
\end{align*}
respectively.
In other words, the number of singular values is $\min(\alpha, \beta)$.

The ``incomplete SVD'' decomposes the matrix as:
\begin{align}
&\mathbb{R}^{\alpha \times \beta} \ni \b{A} = \b{U}\b{\Sigma}\b{V}^\top, \\
&\b{U} \in \mathbb{R}^{\alpha \times k}, ~~ \b{V} \in \mathbb{R}^{\beta \times k}, ~~ \b{\Sigma} \in \mathbb{R}^{k \times k}, \nonumber 
\end{align}
where \cite{golub1970singular}:
\begin{align}
k := \min(\alpha, \beta),
\end{align}
and the columns of $\b{U}$ and the columns of $\b{V}$ are called ``left singular vectors'' and ``right singular vectors'', respectively. 
In incomplete SVD, the $\b{\Sigma}$ is a \textit{square} diagonal matrix whose main diagonal includes the ``singular values''. The matrix $\b{\Sigma}$ is in the form:
\begin{align*}
\b{\Sigma} = 
\begin{bmatrix}
    \sigma_{1} & 0 & 0 \\
    \vdots & \ddots & \vdots \\
    0 & 0 & \sigma_{k} 
\end{bmatrix}.
\end{align*}

Note that in both complete and incomplete SVD, the left singular vectors are orthonormal and the right singular vectors are also orthonormal; therefore, $\b{U}$ and $\b{V}$ are both orthogonal matrices so:
\begin{align}
\b{U}^\top \b{U} = \b{I}, \\
\b{V}^\top \b{V} = \b{I}.
\end{align}
If these orthogonal matrices are not truncated and thus are square matrices, e.g., for complete SVD, we also have:
\begin{align}
\b{U} \b{U}^\top = \b{I}, \\
\b{V} \b{V}^\top = \b{I}.
\end{align}

\begin{proposition}\label{proposition_SVD}
In both complete and incomplete SVD of matrix $\b{A}$, the left and right singular vectors are the eigenvectors of $\b{A}\b{A}^\top$ and $\b{A}^\top \b{A}$, respectively, and the singular values are the square root of eigenvalues of either $\b{A}\b{A}^\top$ or $\b{A}^\top \b{A}$.
\end{proposition}

\begin{proof}
We have:
\begin{align*}
\b{A} \b{A}^\top &= (\b{U}\b{\Sigma}\b{V}^\top) (\b{U}\b{\Sigma}\b{V}^\top)^\top = \b{U}\b{\Sigma}\underbrace{\b{V}^\top \b{V}}_{\b{I}} \b{\Sigma}\b{U}^\top \\
&= \b{U}\b{\Sigma} \b{\Sigma}\b{U}^\top = \b{U}\b{\Sigma}^2\b{U}^\top,
\end{align*}
which is eigen-decomposition \cite{ghojogh2019eigenvalue} of $\b{A} \b{A}^\top$ where the columns of $\b{U}$ are the eigenvectors and the diagonal of $\b{\Sigma}^2$ are the eigenvalues so the diagonal of $\b{\Sigma}$ are the square root of eigenvalues.
We also have:
\begin{align*}
\b{A}^\top \b{A} &= (\b{U}\b{\Sigma}\b{V}^\top)^\top (\b{U}\b{\Sigma}\b{V}^\top) = \b{V}\b{\Sigma}\underbrace{\b{U}^\top \b{U}}_{\b{I}}\b{\Sigma}\b{V}^\top \\
&= \b{V}\b{\Sigma}\b{\Sigma}\b{V}^\top = \b{V}\b{\Sigma}^2\b{V}^\top,
\end{align*}
which is the eigen-decomposition \cite{ghojogh2019eigenvalue} of $\b{A}^\top \b{A}$ where the columns of $\b{V}$ are the eigenvectors and the diagonal of $\b{\Sigma}^2$ are the eigenvalues so the diagonal of $\b{\Sigma}$ are the square root of eigenvalues. Q.E.D.
\end{proof}

\section{Centring the Kernel Matrix for Training and Out-of-sample Data}\label{section_appendix_centeringKernel}

This appendix is based on \cite{scholkopf1997kernel} and Appendix A in \cite{scholkopf1998nonlinear}.

The kernel matrix for the training data, $\{\b{x}_i\}_{i=1}^n$ or $\b{X} \in \mathbb{R}^{d \times n}$, is:
\begin{align}
\mathbb{R}^{n \times n} \ni \b{K} := \b{\Phi}(\b{X})^\top \b{\Phi}(\b{X}),
\end{align}
whose $(i,j)$-th element is:
\begin{align}\label{equation_appendix_kernel_elements}
\mathbb{R} \ni \b{K}(i,j) = \b{\phi}(\b{x}_i)^\top \b{\phi}(\b{x}_j).
\end{align}
We want to center the pulled training data in the feature space:
\begin{align}\label{equation_appendix_centered_pulled_training}
\breve{\b{\phi}}(\b{x}_i) := \b{\phi}(\b{x}_i) - \frac{1}{n} \sum_{k=1}^n \b{\phi}(\b{x}_k).
\end{align}
If we center the pulled training data, the $(i,j)$-th element of kernel matrix becomes:
\begin{align*}
&\breve{\b{K}}(i,j) := \breve{\b{\phi}}(\b{x}_i)^\top \breve{\b{\phi}}(\b{x}_j) \\
&\overset{(\ref{equation_appendix_centered_pulled_training})}{=} \big(\b{\phi}(\b{x}_i) - \frac{1}{n} \sum_{k_1=1}^n \b{\phi}(\b{x}_{k_1})\big)^\top \\
&~~~~~~~~~~~~~~~~~~~~~~~~ \big(\b{\phi}(\b{x}_j) - \frac{1}{n} \sum_{k_2=1}^n \b{\phi}(\b{x}_{k_2})\big) \\
&= \b{\phi}(\b{x}_i)^\top \b{\phi}(\b{x}_j) - \frac{1}{n} \sum_{k_1=1}^n \b{\phi}(\b{x}_{k_1})^\top \b{\phi}(\b{x}_j) \\
&~~~~ - \frac{1}{n} \sum_{k_2=1}^n \b{\phi}(\b{x}_i)^\top \b{\phi}(\b{x}_{k_2}) \\
&~~~~ + \frac{1}{n^2} \sum_{k_1=1}^n  \sum_{k_2=1}^n \b{\phi}(\b{x}_{k_1})^\top \b{\phi}(\b{x}_{k_2}),
\end{align*}
Therefore, the double-centered training kernel matrix is:
\begin{align}
\mathbb{R}^{n \times n} \ni \breve{\b{K}} &= \b{K} - \frac{1}{n} \b{1}_{n \times n} \b{K} - \frac{1}{n} \b{K} \b{1}_{n \times n} \nonumber \\
&~~~~ + \frac{1}{n^2} \b{1}_{n \times n} \b{K} \b{1}_{n \times n} \overset{(\ref{equation_appendix_doubleCentered_A})}{=} \b{H} \b{K} \b{H}, \label{equation_appendix_doubleCentered_training_kernel}
\end{align}
where $\mathbb{R}^{n \times n} \ni \b{1}_{n \times n} := \b{1}_n \b{1}_n^\top$ and $\mathbb{R}^{n} \ni \b{1}_n := [1, \dots, 1]^\top$.

The Eq. (\ref{equation_appendix_doubleCentered_training_kernel}) is the kernel matrix when the pulled training data in the feature space are centered.
In Eq. (\ref{equation_appendix_doubleCentered_training_kernel}), the dimensionality of both centering matrices are $\b{H} \in \mathbb{R}^{n \times n}$.

The kernel matrix for the trainign data and the out-of-sample data, $\{\b{x}_{t,i}\}_{i=1}^{n_t}$ or $\b{X}_t \in \mathbb{R}^{d \times n_t}$, is:
\begin{align}\label{equation_appendix_kernelMatrix_outOfSample}
\mathbb{R}^{n \times n_t} \ni \b{K}_t := \b{\Phi}(\b{X})^\top \b{\Phi}(\b{X}_t),
\end{align}
whose $(i,j)$-th element is:
\begin{align}\label{equation_appendix_kernel_test_elements}
\mathbb{R} \ni \b{K}_t(i,j) = \b{\phi}(\b{x}_{i})^\top \b{\phi}(\b{x}_{t,j}).
\end{align}
We want to center the pulled training data in the feature space, i.e., Eq. (\ref{equation_appendix_centered_pulled_training}). Moreover, the out-of-sample data should be centered using the mean of training (and not out-of-sample) data:
\begin{align}\label{equation_appendix_centered_pulled_outOfSample}
\breve{\b{\phi}}(\b{x}_{t,i}) := \b{\phi}(\b{x}_{t,i}) - \frac{1}{n} \sum_{k=1}^n \b{\phi}(\b{x}_k).
\end{align}
If we center the pulled training and out-of-sample data, the $(i,j)$-th element of kernel matrix becomes:
\begin{align*}
&\breve{\b{K}}_t(i,j) := \breve{\b{\phi}}(\b{x}_i)^\top \breve{\b{\phi}}(\b{x}_{t,j}) \\
&\overset{(a)}{=} \big(\b{\phi}(\b{x}_i) - \frac{1}{n} \sum_{k_1=1}^n \b{\phi}(\b{x}_{k_1})\big)^\top \\
&~~~~~~~~~~~~~~~~~~~~~~~~ \big(\b{\phi}(\b{x}_{t,j}) - \frac{1}{n} \sum_{k_2=1}^n \b{\phi}(\b{x}_{k_2})\big) \\
&= \b{\phi}(\b{x}_i)^\top \b{\phi}(\b{x}_{t,j}) - \frac{1}{n} \sum_{k_1=1}^n \b{\phi}(\b{x}_{k_1})^\top \b{\phi}(\b{x}_{t,j}) \\
&~~~~ - \frac{1}{n} \sum_{k_2=1}^n \b{\phi}(\b{x}_i)^\top \b{\phi}(\b{x}_{k_2}) \\
&~~~~ + \frac{1}{n^2} \sum_{k_1=1}^n  \sum_{k_2=1}^n \b{\phi}(\b{x}_{k_1})^\top \b{\phi}(\b{x}_{k_2}),
\end{align*}
where (a) is because of Eqs. (\ref{equation_appendix_centered_pulled_training}) and (\ref{equation_appendix_centered_pulled_outOfSample}).
Therefore, the double-centered kernel matrix over training and out-of-sample data is:
\begin{align}
\mathbb{R}^{n \times n_t} \ni \breve{\b{K}}_t &= \b{K}_t - \frac{1}{n} \b{1}_{n \times n} \b{K}_t - \frac{1}{n} \b{K} \b{1}_{n \times n_t} \nonumber \\
&~~~~ + \frac{1}{n^2} \b{1}_{n \times n} \b{K} \b{1}_{n \times n_t}, \label{equation_appendix_doubleCentered_outOfSample_kernel}
\end{align}
where $\mathbb{R}^{n \times n} \ni \b{1}_{n \times n} := \b{1}_n \b{1}_n^\top$, $\mathbb{R}^{n \times n_t} \ni \b{1}_{n \times n_t} := \b{1}_n \b{1}_{n_t}^\top$, $\mathbb{R}^{n} \ni \b{1}_n := [1, \dots, 1]^\top$, and $\mathbb{R}^{n_t} \ni \b{1}_{n_t} := [1, \dots, 1]^\top$.

The Eq. (\ref{equation_appendix_doubleCentered_outOfSample_kernel}) is the kernel matrix when the pulled training data in the feature space are centered and the pulled out-of-sample data are centered using the mean of pulled training data.

If we have one out-of-sample $\b{x}_t$, the Eq. (\ref{equation_appendix_doubleCentered_outOfSample_kernel}) becomes:
\begin{align}
\mathbb{R}^{n} \ni \breve{\b{k}}_t &= \b{k}_t - \frac{1}{n} \b{1}_{n \times n} \b{k}_t - \frac{1}{n} \b{K} \b{1}_{n} + \frac{1}{n^2} \b{1}_{n \times n} \b{K} \b{1}_{n}, \label{equation_appendix_doubleCentered_outOfSample_kernel_oneSample}
\end{align}
where:
\begin{align}
&\mathbb{R}^n \ni \b{k}_t = \b{k}_t(\b{X}, \b{x}_t) := \b{\Phi}(\b{X})^\top \b{\phi}(\b{x}_t) \label{equation_appendix_kernelVector_outOfSample} \\
&~~~~~~~~~~~~~~ =[\b{\phi}(\b{x}_1)^\top \b{\phi}(\b{x}_t), \dots, \b{\phi}(\b{x}_n)^\top \b{\phi}(\b{x}_t)]^\top, \nonumber \\
&\mathbb{R}^n \ni \breve{\b{k}}_t = \breve{\b{k}}_t(\b{X}, \b{x}_t) := \breve{\b{\Phi}}(\b{X})^\top \breve{\b{\phi}}(\b{x}_t), \label{equation_appendix_centered_kernelVector_outOfSample} \\
&~~~~~~~~~~~~~~ =[\breve{\b{\phi}}(\b{x}_1)^\top \breve{\b{\phi}}(\b{x}_t), \dots, \breve{\b{\phi}}(\b{x}_n)^\top \breve{\b{\phi}}(\b{x}_t)]^\top, \nonumber 
\end{align}
where $\breve{\b{\Phi}}(\b{X})$ and $\breve{\b{\phi}}(\b{x}_t)$ are according to Eqs. (\ref{equation_appendix_centered_pulled_training}) and (\ref{equation_appendix_centered_pulled_outOfSample}), respectively.

\bibliography{References}
\bibliographystyle{icml2016}

\end{document}